\documentclass[final,12pt]{colt2022} 

\title[Self-Consistency of the Fokker-Planck Equation]{Self-Consistency of the Fokker-Planck Equation}
\usepackage{times}
\author{%
 \Name{Zebang Shen} \Email{zebang@seas.upenn.edu}\\
 \addr University of Pennsylvania
 \AND
 \Name{Zhenfu Wang} \Email{zwang@bicmr.pku.edu.cn} \\
 \addr Peking University
 \AND
 \Name{Satyen Kale} \Email{satyen.kale@gmail.com} \\
 \addr Google
 \AND
 \Name{Alejandro Ribeiro} \Email{aribeiro@seas.upenn.edu} \\
 \addr University of Pennsylvania
 \AND
 \Name{Amin Karbasi} \Email{amin.karbasi@yale.edu} \\
 \addr Yale, Google
 \AND
 \Name{Hamed Hassani} \Email{hassani@seas.upenn.edu} \\
 \addr University of Pennsylvania
}

\input{Required.tex}
\newcommand{\ito}{{It\^o}}
\newcommand{\gronwall}{{Gr\"onwall}}
\newcommand{\ud}{\mathrm{d}}

\newcommand{\udiv}{\mathrm{div}}

\newcommand{\fracd}[2]{\frac{\ud #1}{\ud  #2}}

\begin{document}

\maketitle

\begin{abstract}
	The Fokker-Planck equation (FPE) is the partial differential equation that governs the density evolution of the \ito\ process and is of great importance to the literature of statistical physics and machine learning.
	The FPE can be regarded as a continuity equation where the change of the density is completely determined by a time varying velocity field.
	Importantly, this velocity field also depends on the current density function. As a result, the ground-truth velocity field can be shown to be the solution of a fixed-point equation, a property that we call \textit{self-consistency}. 
	In this paper, we exploit this concept to design a potential function of the hypothesis velocity fields, and prove that, if such a function diminishes to zero during the training procedure, the trajectory of the densities generated by the hypothesis velocity fields converges to the solution of the FPE in the Wasserstein-2 sense.  
	The proposed potential function is amenable to neural-network based parameterization as the stochastic gradient with respect to the parameter can be efficiently computed.
	Once a parameterized model, such as Neural Ordinary Differential Equation is trained, we can generate the entire trajectory to the FPE.
\end{abstract}

\section{Introduction}
We consider the Fokker-Planck equation (FPE) that corresponds to the \ito\ process with a constant diffusion coefficient, which can be written as 
\begin{equation}
	\label{eqn_FPE}
	\fracpartial{}{t}\alpha(t, x) + \udiv \Big(\alpha(t, x) (\underbrace{ - \nabla V(t, x) - \nabla \log\alpha(t, x) }_{\text{underlying velocity field } f^*(t, x)})\Big) = 0, 
\end{equation}
subject to the initial condition 
\begin{equation}
	\label{eqn_FPE_initial}
	\alpha(0, x) = \alpha_0(x).
\end{equation}
Here, $\alpha: [0, T]\times \XM \rightarrow \RBB$ is a time varying density function defined on $\XM\subseteq\RBB^d$, $V: [0, T]\times\XM \rightarrow \RBB$ is a known potential function that determines the drifting term; $\udiv$ and $\nabla$ denote the divergence and gradient operator with respect to the spatial variable $x$ respectively.
The boundary condition that we impose will be introduced in section \ref{section_preliminary}.

FPE is a fundamental problem in the literature of statistical physics due to its wide applications in thermodynamic system analysis \citep{markowich2000trend, lucia2015fokker, qi2016low} and is one of the key equations in the research of the mean field game \citep{cardaliaguet2020introduction,gomes2014mean}.
Recently, it has also been used to model the dynamics of the stochastic gradient descent method on neural networks \citep{chizat2018global,sonoda2019transport,sirignano2020mean,fang2021modeling} and the dynamics of the R\'enyi differential privacy \citep{chourasia2021differential}, and has become a fundamental tool for learning complex distributions and deep generative models due to its deep connection to the Wasserstein gradient flow \citep{sohl2015deep,hashimoto2016learning,liu2019understanding,song2020score,solin2021scalable,mokrov2021large}.
There is a plethora of previous works trying to solve FPE numerically, including the classic mesh-based finite difference and finite volume methods \citep{carrillo2015finite,bailo2018fully}, the stochastic particle methods that are based on the discretization of the Ito SDE \citep{dalalyan2017theoretical,li2019stochastic,li2021sqrt}, the deterministic particle methods that utilize the Gaussian mollifier to approximate the dynamic \citep{degond1990deterministic}, the variational methods that are built on the Wasserstein gradient flow interpretation of the FPE \citep{bernton2018langevin,liu2020neural,carrillo2021primal, ambrosio2005gradient,jko}, and most recently the physics-informed neural network approach that directly parameterize the solution to the FPE and cast the FPE as a root finding problem \citep{han2018solving,long2018pde,long2019pde,raissi2019physics,blechschmidt2021three}.
We note that in all previous approaches, the entity under consideration, i.e. the function to be approximated or learned, is {explicitly} the solution to the PDE \eqref{eqn_FPE}, which is a time-varying probability density function.

In this work, we take a different route:
Instead of approximating the solution to the FPE, we propose to learn the \emph{underlying velocity field} that drives the evolution of the FPE.
The solution to the FPE can then be {implicitly} recovered by the learned velocity field.
Our work is built on a concept called the \emph{self-consistency} of the Fokker-Planck equation: A velocity field that correctly recovers the solution to the FPE should be a fixed point to a \textit{velocity-consistency transformation} (defined in Eq.~\eqref{eqn_transform_A}) derived from the FPE.
The main contribution of our work is summarized as follows.
\begin{quote}
    We establish the theoretical foundation of learning the underlying velocity field of the FPE. Specifically, we design a potential function $R$ for the hypothesis velocity fields $\{f_n\}$ that describes the self-consistency of the Fokker-Planck equation and show that if $R(f_n) \rightarrow 0$ as $n\rightarrow \infty$, the trajectory of distributions generated by $f_\infty$ recovers the solution to the FPE in the Wasserstein-2 sense.
\end{quote}
Moreover, when the hypothesis velocity field is parameterized as a Neural Ordinary Differential Equation $f_\theta$ \citep{chen2018neural}, we discuss how the stochastic gradient of the proposed potential function $R(f_\theta)$ with respect to the parameter $\theta$ of the neural network can be efficiently computed.
Therefore, once $f_\theta$ is trained via stochastic optimization methods, our approach returns an approximate solution to the FPE, which is non-negative and has unit mass, i.e. it integrates to $1$ on $\XM$. These fundamental properties are crucial in real-world physics models and are not guaranteed in previous neural network based approaches.

\section{Preliminaries} \label{section_preliminary}
\paragraph{Boundary Condition}

We assume that the process takes place on a $d$-dimensional box centered around the origin, i.e. $\XM = [-\frac{l}{2}, \frac{l}{2}]^d$.
We consider the periodic boundary condition:
\begin{align}
	\alpha\left(t, (\cdots, -\frac{l}{2}, \cdots)\right) =&\ \alpha\left(t, (\cdots, \frac{l}{2}, \cdots)\right)  \label{eqn_bc_0}
	\\ \fracpartial{}{x} \alpha\left(t, (\cdots, -\frac{l}{2}, \cdots)\right) =&\ \fracpartial{}{x} \alpha\left(t, (\cdots, \frac{l}{2}, \cdots)\right). \label{eqn_bc_1}
\end{align}
The above condition is the same as identifying the points on the corresponding boundaries which happens when the spatial domain is a \emph{torus}.
Note that on a torus, the particle that leaves the torus on the boundary will reenter the domain $\XM$ through the boundary such that $l/2$ (resp., $-l/2$) is replaced by $-l/2$ (resp., $l/2$) in the same coordinate.

The periodic boundary condition (torus) is commonly used in the PDE analysis (e.g. see \citep{JABIN20163588}) with an important technical merit that the integration of a periodic function on the boundary is naturally zero and hence the analysis using integration by parts can be simplified. 
Moreover, it also allows us to focus on the behavior of the PDE system on compact domains without sacrificing the generality, since we can always set the diameter of the torus to be sufficiently large.
We emphasize that to the ML community, this is usually the case of interest: Only in a bounded domain can we expect a neural ODE to be able to represent the underlying velocity field of the FPE, since the neural network is \emph{not} a universal function approximator on unbounded domains. 

In the following, we refer to periodic functions with a period of $l$ as \emph{$l$-periodic}.

\paragraph{Velocity Field and the Induced Push-forward Map}
A velocity field is map $f:[0, T]\times\XM\rightarrow\RBB^d$ that determines the movement of a particle $x(t)$:
\begin{equation} \label{eqn_velocity_field}
	\fracd{}{t}x(t) = f(t, x(t))
\end{equation}
A velocity field $f(t, x)$ induces a push-forward map $X(t, x; f)$ via integrating over time
\begin{equation} \label{eqn_push_forward_map}
	X(t, x_0; f) = x_0 + \int_0^t f(s, x_s) \ud s,
\end{equation}
where $\{x_s\}_{s=0}^t$ is the trajectory of a particle following the velocity field $f(t, x)$ with the initial position $x_0$.
Note that the map $X(t, x; f)$ is invertible under the assumption that $f(t, x)$ is Lipschitz continuous  in $x$ for all $t$. Additionally $X(t, x; f) - x$ is $l$-periodic if we further assume that $f$ is $l$-periodic: For any $i\in\{1,\ldots, d\}$
\begin{equation} \label{eqn_X_x_l_periodic}
	X(t, x_0 + le_i; f) - (x_0 + le_i) = \int_{0}^{t} f(s, x_s + le_i)\ud s = \int_{0}^{t} f(s, x_s)\ud s = X(t, x_0; f) - x_0.
\end{equation}
When the velocity $f$ is clear from the context, we omit the dependence of $X$ on $f$ and write $X(t, x)$, for simplicity.

\paragraph{Neural Ordinary Differential Equation}
The neural ordinary differential equation (NODE) is a favorable instance of the hypothesis class since neural networks are  universal function approximators in a bounded domain and have achieved great recent success in machine learning \citep{chen2018neural,dupont2019augmented,choromanski2020ode}.
Let $f: \RBB \times \RBB^d \times \Theta \rightarrow \RBB^d$ be a neural network parameterized by $\theta\in\Theta$. 
A $d$-dimensional NODE in can be described as
\begin{equation} \label{eqn_NODE}
	\fracd{}{t}x(t) = f_\theta(t, x(t)).
\end{equation}

To accommodate the periodic boundary conditions \eqref{eqn_bc_0} and \eqref{eqn_bc_1}, we need the NODE to be $l$-periodic. 
Consider a $2d$-dimensional NODE with velocity $\tilde f$. We can construct a $d$-dimensional NODE with the following hypothesis velocity field 
\begin{equation}
	f_\theta(t, x(t)) = \tilde f_\theta \left(t, \begin{pmatrix}
		\sin \frac{2\pi}{l} x(t) \\
		\cos \frac{2\pi}{l} x(t)
	\end{pmatrix}\right).
\end{equation}
Here $\sin$ and $\cos$ are applied in an element-wise manner.

\paragraph{Notations}
Consider the $d$-dimensional index vector $a = (a_1,\ldots, a_d)$ with $a_i \in \NBB$ and $\|a\|_1 = k$ and a map $f:\RBB^d \rightarrow \RBB^d$. Denote
\begin{equation} \label{eqn_differentials}
	f^{(a)} = \left[\frac{\partial^{k} f_1}{\partial x_1^{a_1}\ldots \partial x_d^{a_d}}, \cdots, \frac{\partial^{k} f_d}{\partial x_1^{a_1}\ldots \partial x_d^{a_d}}\right],
\end{equation}
where $f_i$ denotes the $i$th entry of $f$.
We define the $k$th order Sobolev norm of a map $f:\XM \rightarrow \RBB^d$ with a base measure $\mu \in \MM_+^1(\XM)$ by
\begin{equation}
	\|f\|_{W^{k, 2}(\mu)} = \left(\sum_{i=0}^{k}\int_\XM \|f^{(i)}(x)\|^2  \mu(x) \ud x \right)^{\frac{1}{2}}.
\end{equation}
Here $f^{(k)} = \{f^{(a)}\}_{a: \|a\|_1 = k}$ denotes the collection of all $k$th order partial derivatives of the map $f$ and is regarded as a $d^{k+1}$-dimensional vector. 
We use $\|\cdot\|$ to denote the spectral norm for matrices and tensors and the standard $\ell_2$-norm for vectors.\\
We use $\{e_i\}$ to denote the standard basis of $\RBB^d$ and use $\Delta$ to denote the Laplacian operator on the spatial variable.
We use $\nabla^i, i\geq 2$ to denote higher order gradient.

\section{Methodology}
Recall that on a torus, when a particle leaves the domain on a boundary, it reappears on the other side (see Figure \ref{fig_torus}-(a)). 
Therefore, the velocity field of the particles are discontinuous on the boundaries, which introduces difficulties in function approximation.
To avoid this issue, a useful and equivalent perspective of the periodic boundary condition is to think of the density function $\alpha(t, \cdot)$ as a $l$-periodic function in every coordinate, i.e. 
\begin{equation} \label{eqn_periodic_measure}
	\forall t, x,\quad \alpha(t, x + le_i) = \alpha(t, x), i \in [d],
\end{equation}
which is depicted in (b) of Figure \ref{fig_torus}.
While particles are allowed to leave $\XM$, the domain of interest, due to the periodicity of the whole domain $\RBB^d$, the total mass within $\XM$ is conserved since the influx and the outflow are balanced.
\begin{figure}[t]
	\centering
	\begin{tabular}{c c}
		\includegraphics[height=0.2\textwidth]{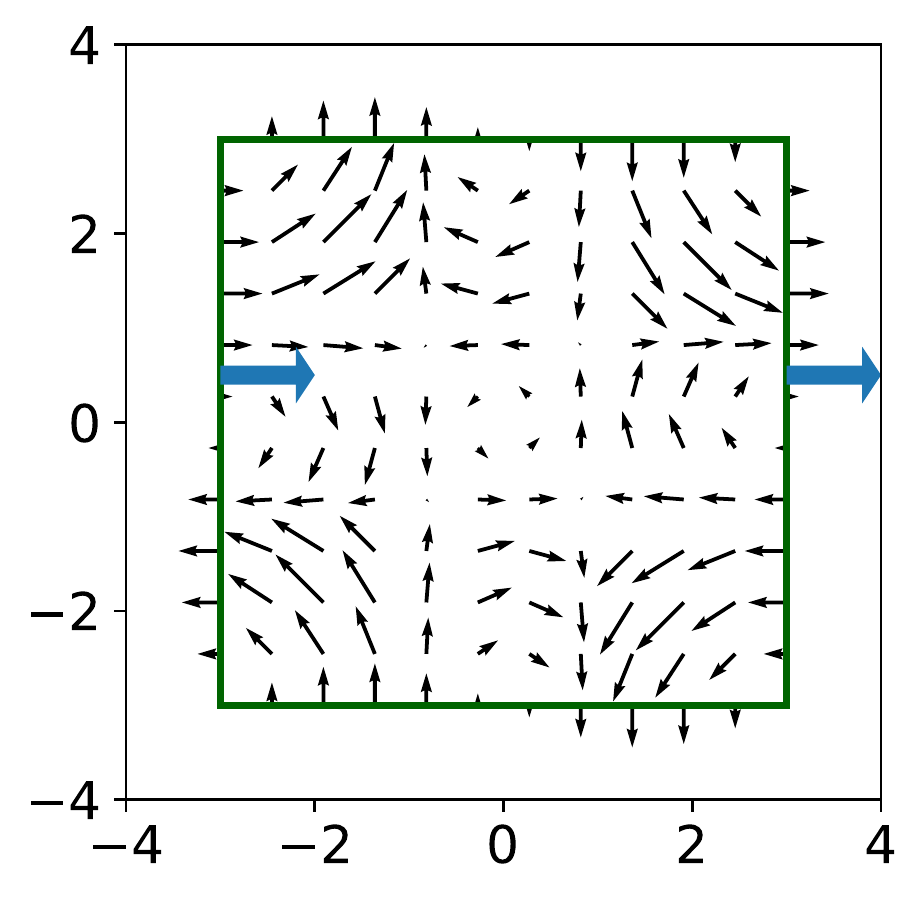} &	
		\includegraphics[height=0.2\textwidth]{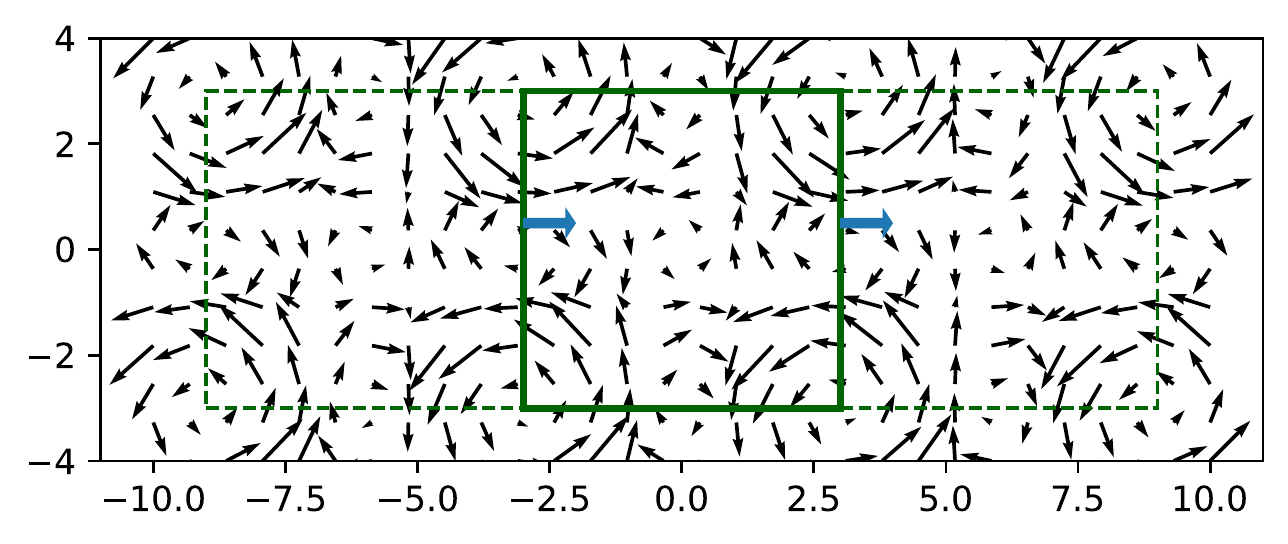} \\
		(a) & (b)
	\end{tabular}
	\caption{Figure (a) depicts that when a particle leaves the torus on a boundary, it reappears on the other side. The velocity field is discontinuous on the boundary. 
		In Figure (b), we consider the periodic extension of the density function $\alpha(t, x)$.
		This is equivalent to the torus since whenever a particle leaves the boundary, another particle will enter $\XM$ from a corresponding adjoining cell.
		Note that in Figure (b) the velocity field is continuous on the whole domain.}
	\label{fig_torus}
\end{figure}
\subsection{Self-consistency of the Fokker-Planck Equation}
Suppose that the particles are distributed initially according to the distribution $\alpha_0$ defined in \eqref{eqn_FPE_initial} and follow a hypothesis velocity field $f(t, x)$. From this perspective, we can write the distribution of particles on $\XM$ at time $t$ in a push-forward manner
\begin{equation} \label{eqn_pushforward_rho1}
	\rho^1(t, \cdot; f) = X(t, \cdot; f) \sharp \alpha_0,
\end{equation}
where the push-forward map $X$, induced by the velocity $f$, is defined in \eqref{eqn_push_forward_map}.
Note that $\rho^1$ is well-defined on the whole domain $\RBB^d$, but we restrict our interest to $\XM$.
Based on this notation, the Fokker-Planck equation \eqref{eqn_FPE} induces a {\it velocity-consistency transformation} $\AM$ of the velocity field in the following manner:
\begin{equation} \label{eqn_transform_A}
	\AM[f](t, x) = -\nabla V(t, x) - \nabla \log \rho^1(t, x; f).
\end{equation}
Observe that, for the ground-truth velocity field $f^*$ that drives the particle evolution of the Fokker-Planck equation, i.e. $f^*(t, x) = - \nabla V(t, x) - \nabla \log \alpha(t, x)$, we have $$\AM[f^*] = f^*.$$ We term this property the \emph{self-consistency} of the Fokker-Planck equation. Similar to Eq.~\eqref{eqn_pushforward_rho1}, we can define $\rho^2(t, \cdot; f) = X(t, \cdot; \AM[f]) \sharp \alpha_0,$. Indeed, the interplay between the two systems $\rho^1$ and $\rho^2$ is crucial to our analysis.

The goal of our paper is to show that if a sequence of hypothesis velocity fields $\{f_n\}$ asymptotically satisfies the above consistency property, i.e. $\|\AM[f_n] - f_n\| \rightarrow 0$ as $n\rightarrow\infty$ for some appropriate norm $\|\cdot\|$, then the distribution $\rho^1(t, x; f_\infty)$ generated from the hypothesis velocity field $f_\infty$ recovers $\alpha(t, x)$, the solution to the FPE \eqref{eqn_FPE} in the Wasserstein-2 sense.

\subsection{Designing the Self-Consistency Potential Function and its Computation}
Given a hypothesis velocity field $f$, we denote the difference between $f$ and $\AM[f]$ by
\begin{equation} \label{eqn_delta_def}
	\delta(t, x; f) = f(t, x) - \AM[f](t, x).
\end{equation}
We propose to use the time average of the $2$nd order Sobolev norm of $\delta$ with the base measure $\rho^{1}(t, \cdot; f)$ as the potential function of $f$:
\begin{align*}
	R(f) = \int_{0}^{T} \int_\XM\sum_{i=0}^2 \|\delta^{(i)}(t, x; f)\|^2 \rho^{1}(t, x; f)\ud x \ud t
	= \int_{0}^{T} \| \delta(t, \cdot; f) \|^2_{W^{2, 2}(\rho^1(t, \cdot; f))} \ud t.
\end{align*}
In Section \ref{section_analysis}, we show that $R(f)$ controls the Wasserstein-2 distance between $\rho^1(t, \cdot)$ and $\alpha(t, \cdot)$, i.e. for any time $t\in[0, T]$, $W_2^2(\rho^1(t, \cdot), \alpha(t, \cdot)) = O(R(f))$.
This result has two direct implications: (i) Given a hypothesis velocity field $f$, we can use $R(f)$ to measure its quality in terms of recovering the solution to the FPE; (ii) Given a class of parameterized hypothesis velocity fields $f_\theta$, one can find the best parameter $\theta$ by minimizing $R(f_\theta)$ with a \emph{learning} procedure, which is discussed in details at the end of this section.
By ``learning", we mean to distinguish our approach from the previous numerical FPE solvers, e.g. the JKO method, which are in essence ``simulating" the FPE dynamics: They iteratively update the configuration of the system using certain rules derived from the FPE.
In contrast, the proposed potential function describes the self-inconsistency of a hypothesis velocity field, which can be refined through a training procedure.

%
%

The potential function $R(f)$ might seem difficult to compute at first. In the following, we present an equivalent formulation of $R(f)$ from the perspective of particle trajectory, which is critical to our analysis and to the actually computation of $R(f)$.
We first introduce the following important change-of-variables formula of integrating periodic functions on $\XM$. 
Recall that the standard change-of-variables formula reads as follows: for a function $g:\RBB^d \rightarrow \RBB$
\begin{equation}
	\int_\XM g \ud X\sharp \alpha = \int_{X^{-1}(\XM)} g\circ X \ud \alpha.
\end{equation}
In brief, we show that for an $l$-periodic functions $g$ the integration domain $X^{-1}(\XM)$ on the RHS of the above equation can be replaced by $\XM$.
The proof is deferred to Appendix \ref{appendix_proof_lemma_change_of_variable}.
\begin{lemma} \label{lemma_change_of_variables}
	Consider an invertible mapping $X: \RBB^d \rightarrow \RBB^d$ such that $X(x) - x$ is $l$-periodic, an $l$-periodic function $g: \RBB^d \rightarrow \RBB$, and an $l$-periodic measure $\alpha$. The following formula holds:
	\begin{equation}
		\int_\XM g \ud X\sharp \alpha = \int_\XM g\circ X \ud \alpha,
	\end{equation}
	where $\XM$ is the centered $d$-dimensional box defined above.
\end{lemma}
Note that the push-forward map $X$ defined in \eqref{eqn_push_forward_map} is invertible and $X(t, x; f) - x$ is $l$-periodic (see \eqref{eqn_X_x_l_periodic}), the integrand in $R(f)$ is $l$-periodic, and from \eqref{eqn_periodic_measure} the measure $\alpha_0$ is also $l$-periodic.
Using the above lemma, we have
\begin{align}
	R(f) =&\ \int_{0}^{T} \int_\XM\sum_{i=0}^2 \|\delta^{(i)}(t, \cdot; f)\|^2 \ud X(t, \cdot; f)\sharp \alpha_0 \ud t\\
	=&\ \int_{0}^{T} \int_\XM\sum_{i=0}^2 \|\delta^{(i)}(t, X(t, x; f); f)\|^2 \ud \alpha_0(x) \ud t
\end{align}
If we further define the trajectory-wise loss
\begin{equation}
	R(f; x_0) = \int_{0}^{T} \sum_{i=0}^2 \|\delta^{(i)}\left(t, X(t, x_0; f); f\right)\|^2 \ud t,
\end{equation}
the potential function $R(f)$ admits an equivalent formulation
\begin{equation}
	R(f) = \int_{\XM} R(f; x_0) \alpha_0(x_0) \ud x_0.
\end{equation}
Therefore, we have that $R(f; x_0)$ is an unbiased estimator of the objective $R(f)$.
In the following, we elaborate on how $R(f; x_0)$ can be computed.
\paragraph{Computation of the trajectory-wise loss $R(f; x_0)$}
We now discuss how the function $R(f; x_0)$ can be computed.
We assume that we have the exact expression of $f$ and $V$, and hence we can readily evaluate $f^{(i)}$ for $i \in \{0, 1, 2\}$ and $V^{(i)}$ for $i \in \{1, 2, 3\}$ (recall the notation of differentials in \eqref{eqn_differentials}).
Use $x(t) = X(t, x_0; f)$ to denote the trajectory of a particle with the initial position $x_0$ and following the velocity field $f$.
In the following, we address how $\nabla^i \log \rho^{1}(t, x(t); f)$ for $i\in\{1, 2, 3\}$ can be computed since these are the only unknown terms when evaluating $\AM[f]^{(i)}(t, x(t))$ for $i \in \{0, 1, 2\}$.
The proofs of the following propositions are deferred to the appendix. 
We first compute the first order gradient of the log-probability.
\begin{proposition} \label{prop_score}
	Denote $f_{t}(x) = f(t, x)$ and $\rho^{1}_{t} = \rho^{1}(t, x; f)$ where we recall that $\rho^{1}(t, x; f)$ is the density function formally defined in equation \eqref{eqn_pushforward_rho1}. We have
	\begin{equation*}
		\fracd{ }{t}\nabla \log \rho_{t}^{1}(x(t)) = - \nabla  \udiv\left(f_{t}(x(t))\right) -   \left(\nabla{f_{t}}(x(t))\right)^\top \nabla \log \rho_{t}^{1}(x(t)),
	\end{equation*}
\end{proposition}
The second order partial derivatives of the log-probability is computed as follows.  
\begin{proposition} \label{prop_Hessian_log}
	Denote $f_{t}(x) = f(t, x)$ and $\rho_{t}^{1} = \rho^{1}(t, x; f)$ where we recall that $\rho^{1}(t, x; f)$ is formally defined in equation \eqref{eqn_pushforward_rho1}.
	The time evolution of the 2nd order gradient of the log probability function can be computed by 
	\begin{align*}
		\fracd{ }{ t}\frac{\partial^2}{\partial x_i \partial x_j} \log \rho_{t}^{1}(x(t)) = -\frac{\partial^2}{\partial x_i \partial x_j} \udiv f_{t}(x_t)  - \fracpartial{}{x_i}\nabla \log\rho_{t}^{1}(x(t))\cdot \fracpartial{}{x_j} f_{t}(x(t)) \qquad \notag \\
		-  \fracpartial{}{x_i} f_{t}(x(t)) \cdot \fracpartial{}{x_j}\nabla \log\rho_{t}^{1}(x(t)) -  \frac{\partial^2}{\partial x_i \partial x_j} f_{t}(x(t)) \cdot \nabla\log\rho_{t}^{1}(x(t)).
	\end{align*}
\end{proposition}
The third order partial derivatives of the log-probability is computed as follows. 
\begin{proposition} \label{prop_3rd_gradient_log}
	Denote $f_{t}(x) = f(t, x)$ and $\rho_{t}^{1} = \rho^{1}(t, x; f)$ where we recall that $\rho^{1}(t, x; f)$ is formally defined in equation \eqref{eqn_pushforward_rho1}.
	The time evolution of the 3rd order gradient of the log probability function can be computed by 
	\begin{align*}
		 \fracd{ }{ t} \frac{\partial^3}{\partial x_i \partial x_j \partial x_k}& \log \rho_{t}^{1}(x(t)) = - \frac{\partial^3}{\partial x_i \partial x_j \partial x_k} \udiv f_t(x(t)) - \frac{\partial^2}{\partial x_i \partial x_j} \nabla \log \rho_{t}^1(x(t)) \cdot \frac{\partial}{\partial x_k} f_{t}(x(t)) \\
		&\quad \quad-\frac{\partial^2}{\partial x_i \partial x_k}\nabla \log \rho_{t}^1(x(t)) \cdot \partial_{j} f_{t}(x(t)) - \frac{\partial}{\partial x_i} \nabla \log \rho_{t}^1(x(t)) \cdot \partial_{j, k} f_{t}(x(t)) \\
		&\quad \quad-\frac{\partial^2}{\partial x_j \partial x_j}\nabla \log \rho_{t}^1(x(t)) \cdot \partial_{i} f_{t}(x(t)) - \frac{\partial}{ \partial x_j} \nabla \log \rho_{t}^1(x(t)) \cdot \partial_{i, k} f_{t}(x(t)) \\
		&\quad \quad-\frac{\partial}{\partial x_k}\nabla \log \rho_{t}^1(x(t)) \cdot \partial_{i,j} f_{t}(x(t)) - \nabla \log \rho_{t}^1(x(t)) \cdot \frac{\partial^3}{\partial x_i \partial x_j \partial x_k} f_{t}(x(t)).
	\end{align*}
\end{proposition}
The above propositions show that the evolution of the $i$th order differential of $\log \rho^1_t$ only depends differentials with order no more than $i$.
This means that the differentials of $\log \rho^1_t$ can be exactly computed using only local information, even though they depend on the macroscopic distribution. Note that this is only possible along $\{x(t)\}$, the trajectory of the particle under consideration.

\paragraph{Parameterizing the Hypothesis Velocity Field with NODE}
In the following, we take the NODE as a specific parameterized instance of the hypothesis velocity field $f_\theta$.
Recall that $R(f_\theta; x_0)$ is an unbiased estimator of $R(f_\theta)$.
A key step in the optimization of a neural network is to compute the stochastic gradient $\nabla_\theta R(f_\theta; x_0)$, which is elaborated as follows.

Suppose that the initial point $x_0$ is fixed.
To compute $\nabla_\theta R(f_\theta; x_0)$, the gradient of the trajectory-wise loss with respect to the parameter $\theta$, we write $R(f_\theta; x_0)$ in a standard ODE-constrained form:
\begin{equation} \label{eqn_trajectory_wise_inconsistency_ODE_loss}
	R(f_\theta; x_0) = \ell(\theta)\ \defi \int_0^T g(t, s(t), \theta) \ud t
\end{equation}
where $\{s(t)\}_{t\in[0, T]}$ is the solution to the ODE
\begin{equation} \label{eqn_trajectory_wise_inconsistency_ODE}
	\begin{cases}
		\frac{\ud }{\ud t} s(t) = \psi(t, s(t); \theta) \\
		s(0) = s_0(x_0).
	\end{cases}
\end{equation}
Recall the definition of the differentials $f^{(i)}$ in \eqref{eqn_differentials}.
Here, the time-varying state $s(t)$ is 
\begin{equation}
	s(t) = [x(t), \zeta_1(t), \zeta_2(t), \zeta_3(t)],
\end{equation}
where $\zeta_i(t) = (\log \rho^1)^{(i+1)}(t, x(t); f_\theta)$ for $i \in \{0, 1, 2\};$ 
$s_0$ is a function of $x_0$
\begin{equation}
	s_0(x_0) = [x_0,  (\log \alpha_0)^{(1)}(x_0), (\log \alpha_0)^{(2)}(x_0), (\log \alpha_0)^{(3)}(x_0)];
\end{equation}
Here, $\psi$ is the velocity field that drives the evolution of the state $s$ such that the first component of $s$ is updated according to the ODEs in equation \eqref{eqn_NODE} and the last three components of $s$ are updated according to propositions \ref{prop_score} to \ref{prop_3rd_gradient_log} respectively; 
and the function $g$ is define as
\begin{align*}
	g(t, s(t); \theta) = \sum_{i=0}^2\|f^{(i)}_\theta(t, x(t)) + V^{(i+1)} (t, x(t))+ \zeta_i\|^2,
\end{align*}
so that we recover the difference function $\delta$ defined in \eqref{eqn_delta_def}.
Note that by introducing the auxiliary states $\zeta_i$, the function $g$  depends on $\theta$ only through $f^{(i)}_\theta(t, x(t))$.
With the above standard ODE-constrained form of $R(f_\theta; x_0)$, we can compute $\nabla \ell(\theta)$ in equation \eqref{eqn_trajectory_wise_inconsistency_ODE_loss} using the classic adjoint method, which is provided in Appendix \ref{appendix_adjoint_method}.

\paragraph{Recovering an Approximate Solution to the FPE}
Given a hypothesis velocity field $f$, we return $\rho(t, \cdot; f)$ as an approximate solution to the FPE.
To evaluate $\rho(t, x; f)$ for any $x\in\XM$, let ${x(s)}_{s \in [0, t]}$ be the trajectory of the final value problem
\begin{equation}
	\frac{\ud x(s)}{\ud s} = f(s, x(s)), x(t) = x.
\end{equation}
We can compute that $\fracd{ }{ t} \log \rho^{1}(t, x(t); f) 	=  \fracpartial{}{t} \log \rho^{1}(t, x(t); f) + f(t, x(t))\cdot \nabla \log \rho^{1}(t, x(t); f)$.
Using the FPE \eqref{eqn_FPE}, we derive
$\fracpartial{}{t} \log \rho^{1}(t, x) = -\udiv f(t, x) - \nabla \log \rho^1(t, x; f) \cdot f(t, x)$,
and hence we have $\fracd{ }{ t} \log \rho^{1}(t, x(t); f) = -\udiv f(t, x(t))$.
Therefore, we can compute $\log \rho^1(t, x; f)$ by
\begin{equation} \label{eqn_recover_density_from_velocity}
    \log \rho^1(t, x; f) = \log \alpha_0(x(0)) - \int_0^t \udiv f(t, x(s)) \ud s.
\end{equation}
\section{Analysis} \label{section_analysis}
In this section, we prove that the potential function $R(f)$ inspired by the self-consistency of the FPE controls the Wasserstein-2 distance between $\rho^1(t, \cdot; f)$ and $\alpha(t, \cdot)$ for all $t\in[0, T]$.
We achieve this by introducing an auxiliary distribution $\rho^2$ induced by $\AM[f]$ to bridge the hypothesis distribution $\rho^1(t, \cdot; f)$ induced by the velocity field and the solution to the FPE $\alpha(t, \cdot)$.
This allows us to control the Wasserstein-2 distance between $\rho^1(t, \cdot; f)$ and $\rho^2(t, \cdot; f)$ and the KL-divergence between $\rho^2(t, \cdot; f)$ and $\alpha(t, \cdot)$ separately.
We first present the assumptions required for our analysis.
\begin{assumption}[Regularity of the initial distribution] \label{ass_regularity_0}
	For any $x\in\XM$, the Hessian of the {log probability} of the initial distribution $\rho^{1}_0 = \alpha_0$ is bounded, i.e.
	\begin{equation}
		\max\{\|\nabla \log \alpha_0(x)\|, \|\nabla^2 \log \alpha_0(x)\|, \|\nabla^3 \log \alpha_0(x)\|, \|\nabla^2 \Delta \log \alpha\|\}\leq L_0.
	\end{equation}
\end{assumption}
\begin{assumption}[Regularity of the hypothesis velocity field] \label{ass_regularity_of_f}
	The hypothesis velocity field $f$ is $l$-periodic for any time $t$ and parameter $\theta$.
	Moreover, given a fixed time horizon $T>0$ of the evolution, for any space-time variables $x\in\XM$ and $t\in[0, T]$ and any neural network parameters $\theta\in\Theta$, 
 	the hypothesis velocity field $f$ in NODE satisfies that for all $x\in\XM$
	\begin{equation}
		\max \{\|\max_{i\in\{1, 2, 3, 4\}} \nabla^i f_{t}(x)\|, \max_{i\in\{1, 2, 3, 4\}}\|\nabla^i \udiv f_{t}(x)\|\} \leq L_f.
	\end{equation}
\end{assumption}
\begin{assumption}[Regularity of the drifting term]\label{ass_regularity_of_v}
	For all $t$, the potential function $V(t, \cdot)$ is $l$-periodic and for all $x\in\XM$
		$\max\{\|\nabla^2 V(t, x)\|, \|\nabla^3 V(t, x)\|, \|\nabla^2 \Delta V(x)\|\}\leq L_v.$
\end{assumption}
We state our main result as follows.
\begin{theorem}[main result] \label{thm_main}
	Suppose that the assumptions \ref{ass_regularity_0} to \ref{ass_regularity_of_v} hold. 
	We have for all $t\in[0, T]$
	\begin{equation}
		W_2^2(\rho^1(t, \cdot; f), \alpha(t,\cdot)) \leq d\,l \,c \cdot R(f),
	\end{equation}
	where $l$ is length of the box $\XM$, $d$ is the dimension of the ambient space, and $c$ is a constant that depends on the regularity constants $L_0$, $L_f$, $L_v$ and the maximum evolving time $T$.
\end{theorem}
The following corollary states that if we can optimize over the hypothesis velocity field $f$ such that $R(f)$ diminishes to zero, we can recover the solution to the FPE in the Wasserstein-2 sense.
\begin{corollary}
	Suppose that assumptions \ref{ass_regularity_0} to \ref{ass_regularity_of_v} hold and assume a sequence of hypothesis velocity fields ${f_n}$ satisfies $R(f_n) \rightarrow 0$ as $n\rightarrow \infty$.
	We have $W_2^2(\rho^1(t, \cdot; f_n), \alpha(t,\cdot))\rightarrow 0$ as $n\rightarrow \infty$.
\end{corollary}
\begin{remark}
	Assume that the class of hypothesis velocity fields is the NODE $f_\theta$ (see \eqref{eqn_NODE}). Also, assume  that the underlying velocity field $f^*$ is sufficiently regular such that it can be represented by $f_{\theta^*}$ for some optimal parameter $\theta^*$.
	Then, we can optimize over the parameter $\theta$ and recover the solution to the FPE if $R(f_\theta)$ diminishes to zero during the training phase.
\end{remark}
We now present the proof of Theorem \ref{thm_main} which is built on the interplay between two systems:
The first is described by the hypothesis velocity field $f$:
\begin{equation} \label{eqn_system_i}
	\text{System (1): } \frac{\ud x(t)}{ \ud t} = f(t, x(t));
\end{equation}
and the second is driven by $\AM[f]$ which is defined in \eqref{eqn_transform_A}:
\begin{equation} \label{eqn_system_ii}
	\text{System (2): } \frac{\ud y(t)}{\ud t} = \AM[f](t, y(t)).
\end{equation}
Similar to the push-forward map $X(t, \cdot; f)$ defined in \eqref{eqn_push_forward_map}, $\AM[f]$ also induces a map $X(t, \cdot; \AM[f])$.
To better distinguish these two systems, we denote $Y(t, x; f) = X(t, x; \AM[f])$ and define $\rho^2(t, \cdot; f) = Y(t, \cdot; f)\sharp\alpha_0$ for system (2).
These notations are summarized in Table \ref{table_ss}.

\begin{table}[t]
	\caption{Summary of the notations for systems (1) and (2). Note that $Y(t, \cdot; f) = X(t, \cdot; \AM[f])$.}
	\centering
	\begin{tabular}{c | c | c | c | c}
	\hline
		& velocity field & particle map & particle trajectory & density  \\ \hline
		System (1) & $f(t, x)$ & $X(t, \cdot; f)$ & $\{x(t)\}$ & $\rho^{1}(t, \cdot; f) = X(t, \cdot; f)\sharp \alpha_0$ \\ \hline
		System (2) & $\AM[f](t, x)$ & $Y(t, \cdot; f)$ & $\{y(t)\}$ & $\rho^{2}(t, \cdot; f) = Y(t, \cdot; f)\sharp\alpha_0$
	\end{tabular}
	\label{table_ss}
\end{table}
The following lemma establishes some regularity results of the involved velocity fields.
\begin{lemma} \label{lemma_regulairty}
	Recall Systems (1) and (2) in Table \ref{table_ss}.
	For simplicity of notations, denote their probability density functions by $\rho_{t}^{1}$ and $\rho^{2}_{t}$ respectively.
	Additionally, we denote $f_{t}(x) = f(t, x)$ and $\AM[f]_t(x) = \AM[f](t, x)$.
	We have that for all $t \in [0, T]$
	\begin{enumerate}
		\item Both $\rho_{t}^{1}$ and $\rho_{t}^{2}$ are $l$-periodic.
		\item $\nabla \log \rho_{t}^1$ is bounded and Lipschitz continuous.
		\item Both $\nabla \AM[f]_t$ and $\nabla \udiv \AM[f]_t$ are bounded and Lipschitz continuous.
	\end{enumerate}
\end{lemma}
The following lemma shows that $R(f)$ controls the Wasserstein-2 distance between $\rho^1(t, \cdot)$ and $\rho^2(t, \cdot)$ for all $t\in[0, T]$.
The full proof is provided in Appendix \ref{appendix_proof_sys_1_2}.
\begin{lemma} \label{thm_sys_1_2}
	Recall Systems (1) and (2) in Table \ref{table_ss}.
	Denote their probability density functions by $\rho_{t}^{1}$ and $\rho_{t}^{2}$ respectively.
	Under Assumptions \ref{ass_regularity_0} to \ref{ass_regularity_of_v}, there exists a constant $C_1$ such that
	\[
	\sup_{t \in [0, T]}	W_2^2(\rho_{t}^{1}, \rho_{t}^{2}) \leq C_1 R(f),
	\]
	where $C_1$ depends on the maximum evolving time $T$ and $L_0$, $L_f$, $L_v$ defined in assumptions \ref{ass_regularity_0} to \ref{ass_regularity_of_v}.
\end{lemma}
\begin{proof}[A sketch of the proof.]
    We first note that $P(t, \cdot; f) = Y(t, \cdot; f)\circ X(t, \cdot; f)^{-1}$ is a transport map such that $\rho^2(t, \cdot; f) = P(t, \cdot; f)\sharp \rho^1(t, \cdot; f)$.
    Consequently, from the definition of the Wasserstein-2 distance, we have
	\begin{align*}
		W_2^2(\rho_{t}^{1}, \rho_{t}^{2}) \leq&\ \int_\XM \|x - P(t, x; f)\|^2 \ud \rho^{1}(t, x; f) 
		= \int_\XM \|X(t, x; f) - Y(t, x; f)\|^2 \ud \alpha_0(x) \\
		=&\ \int_\XM \|x(t) - y(t)\|^2 \ud \alpha_0(x_0),
	\end{align*}
	where we used the change-of-variables formula of the push-forward measure from Lemma \ref{lemma_change_of_variables} in the first equality and $\{x(t)\}_{t\in[0, T]}$ and $\{y(t)\}_{t\in[0, T]}$ are the trajectory of particles initialized from $x_0$ but driven by Systems (1) and (2) respectively.
	We then study the dynamic of $\fracd{}{t}\|x(t) - y(t)\|^2$ and prove the lemma using the \gronwall's inequality.
\end{proof}	

We then show that $R(f)$ controls the distance between score functions of systems (1) and (2).
The full proof is provided in Appendix \ref{appendix_proof_lemma_score}.
\begin{lemma} \label{lemma_main}
	Recall Systems (1) and (2) in Table \ref{table_ss}. For simplicity of notations, denote their probability density functions by $\rho_{t}^{1}$ and $\rho_{t}^{2}$ respectively.
	Denote the weighted $L_2$ norm by
	\begin{equation}
		\xi_t\ \defi\ \|\nabla\log\rho_{t}^{1} - \nabla\log\rho_{t}^{2}\|^2_{\rho_{t}^{2}}.
	\end{equation}
	Suppose assumptions \ref{ass_regularity_0} to \ref{ass_regularity_of_v} hold. 
	There exists some constant $C_2$ such that for any $t\in[0, T]$, 
	\begin{equation}
		\int_0^t \xi_s \ud s \leq C_2 R(f),
	\end{equation}
	where $C_2$ depends on the maximum evolving time $T$ and $L_0$, $L_f$, $L_v$ defined in assumptions \ref{ass_regularity_0} to \ref{ass_regularity_of_v}.
\end{lemma}
\begin{proof}[A sketch of the proof.]
    With the change-of-variables lemma, we can expand
    \begin{align*}
		\xi_t =&\ \|\nabla\log\rho_{t}^{1}\circ Y_{t} - \nabla\log\rho_{t}^{2}\circ Y_{t}\|_{\alpha_0}^2 \\
		\leq&\ \|\nabla\log\rho_{t}^{1}(y(t)) - \nabla\log\rho_{t}^{1}(x(t))\|^2_{\alpha_0} + \|\nabla \log \rho_{t}^{1}(x(t)) - \nabla \log \rho_{t}^{2}(y(t))\|^2_{\alpha_0}.
	\end{align*}
	The first term can be control by the Lipschitz continuity of $\nabla \log \rho_t^1$.
	We study the dynamic of the second term using Proposition \ref{prop_score} and use the \gronwall's inequality to establish the lemma.
\end{proof}

Built on the above two lemmas, the following lemma states the most novel part of our analysis which shows that the KL-divergence between $\rho^2(t, \cdot)$ generated by system (2) and the solution to the FPE $\alpha(t, \cdot)$ is controlled by $R(f)$.
\begin{lemma} \label{lemma_kl}
	Recall Systems (1) and (2) in Table \ref{table_ss}. For simplicity of notations, denote their probability density functions by $\rho_{t}^{1}$ and $\rho_{t}^{2}$ respectively and use $\alpha_t$ to denote the solution to the Fokker-Planck equation \eqref{eqn_FPE}.
	Suppose assumptions \ref{ass_regularity_0} to \ref{ass_regularity_of_v} hold.
	For any $t\in[0, T]$, we have
	\begin{equation}
		\mathrm{KL}(\rho_{t}^{2}, \alpha_t) \leq \frac{C_2}{2} R(f),
	\end{equation}
	where $C_2$ is the constant defined in Lemma \ref{lemma_main}.
\end{lemma}
\begin{proof}
	We study the evolution of the KL divergence between $\rho_{t}^{2}$ and $\alpha_t$.
	Recall that for $\rho_{t}^{2}$, we have
	\begin{equation} \label{eqn_pde_rho_2_t}
		\fracpartial{\rho_{t}^{2}}{t} = \Delta \rho_{t}^{2} + \udiv\left(\rho_{t}^{2} \nabla V_t \right) + \udiv(\rho_{t}^{2} \nabla \log \frac{\rho_{t}^{1}}{\rho_{t}^{2}} ) ,
	\end{equation}
	and for $\alpha_t$ we have
	\begin{equation} \label{eqn_pde_alpha_t}
		\fracpartial{\alpha_t}{t} = \Delta \alpha_t + \udiv\left(\alpha_t \nabla V_t \right) .
	\end{equation}
	We can compute
	\begin{align}
		\fracd{\mathrm{KL}(\rho_{t}^{2}, \alpha_t)}{t} = \int_{\XM} \fracpartial{\rho_{t}^{2}}{t} \log\frac{\rho_{t}^{2}}{\alpha_t} + \fracpartial{\rho_{t}^{2}}{t} - \rho_{t}^{2} \fracpartial{\log \alpha_t}{t}\ud x \notag 
		= \int_{\XM} \fracpartial{\rho_{t}^{2}}{t} \log\frac{\rho_{t}^{2}}{\alpha_t} - \frac{\rho_{t}^{2}}{\alpha_t} \fracpartial{\alpha_t}{t}\ud x,
	\end{align}
	where in the second equality, we use $\int_{\XM} \fracpartial{\rho_{t}^{2}}{t}\ud x = \fracd{\int_{\XM} \rho_{t}^{2} \ud x }{t} = \fracd{1}{t} = 0$.
	Plug \eqref{eqn_pde_rho_2_t} and \eqref{eqn_pde_alpha_t} in the above equation to derive
	\begin{align}
		\notag \fracd{\mathrm{KL}(\rho_{t}^{2}, \alpha_t)}{t} =&\ \int_{\XM} \Delta \rho_{t}^{2} \log\frac{\rho_{t}^{2}}{\alpha_t} + \udiv\left(\rho_{t}^{2} \nabla V_t \right) \log\frac{\rho_{t}^{2}}{\alpha_t} + \udiv(\rho_{t}^{2} \nabla \log \frac{\rho_{t}^{1}}{\rho_{t}^{2}}) \log\frac{\rho_{t}^{2}}{\alpha_t}\ud x \\
		&\ \qquad - \int_\XM \frac{\rho_{t}^{2}}{\alpha_t} \Delta \alpha_t + \frac{\rho_{t}^{2}}{\alpha_t} \udiv\left(\alpha_t \nabla V_t \right) \ud x. \label{eqn_dkl_dt}
	\end{align}
	Combine the first and fourth terms of the above equation.
	Using integration by part, we have that
	\begin{align}
		\notag \int_{\XM} \Delta \rho_{t}^{2} \log\frac{\rho_{t}^{2}}{\alpha_t} - \frac{\rho_{t}^{2}}{\alpha_t} \Delta \alpha_t \ud x =&\ \int_{\XM} - \nabla\rho_{t}^{2}\cdot\nabla\log\frac{\rho_{t}^{2}}{\alpha_t} + \nabla\frac{\rho_{t}^{2}}{\alpha_t}\cdot \nabla \alpha_t \ud x \\
		\notag=&\ \int_{\XM} - \nabla\rho_{t}^{2}\cdot\nabla\log\frac{\rho_{t}^{2}}{\alpha_t} + \frac{\nabla \rho_{t}^{2} \alpha_t - \rho_{t}^{2} \nabla\alpha_t}{(\alpha_t)^2}\cdot \nabla \alpha_t \ud x \\
		=&\ - \|\nabla \log \rho_{t}^{2} - \nabla \log \alpha_t\|^2_{\rho_{t}^{2}},
	\end{align}
	where we note that the integration on the boundary $\partial \XM$ is 0 due to the periodic boundary conditions \eqref{eqn_bc_0} and \eqref{eqn_bc_1}.
	Combine the second and the last terms of \eqref{eqn_dkl_dt}.
	Use integration by part to compute
	\begin{align*}
		\notag&\ \int_\XM\udiv\left(\rho_{t}^{2} \nabla V_t \right) \log\frac{\rho_{t}^{2}}{\alpha_t} - \frac{\rho_{t}^{2}}{\alpha_t} \udiv\left(\alpha_t \nabla V_t \right) \ud x\\
		\notag =&\ \int_\XM- \left(\rho_{t}^{2} \nabla V_t \right)\cdot \nabla\log\frac{\rho_{t}^{2}}{\alpha_t} + \frac{\nabla \rho_{t}^{2} \alpha_t - \rho_{t}^{2} \nabla\alpha_t}{(\alpha_t)^2} \cdot\left(\alpha_t \nabla V_t \right) \ud x = 0\\
	\end{align*}
	\vspace{-3mm}
	We hence have (using integration by part for the third term of \eqref{eqn_dkl_dt})
	\begin{align*}
		\notag \fracd{\mathrm{KL}(\rho_{t}^{2}, \alpha_t)}{t} = - \|\nabla \log \rho_{t}^{2} - \nabla \log \alpha_t\|^2_{\rho_{t}^{2}} 
		&\ - \int_{\XM}\rho_{t}^{2} \nabla \log \frac{\rho_{t}^{1}}{\rho_{t}^{2}} \cdot \nabla\log\frac{\rho_{t}^{2}}{\alpha_t}\ud x.
	\end{align*}
	Using $2(a - b)(a - c) = \|a-b\|^2 + \|a-c\|^2 - \|b -c\|^2$, we have
	\begin{align*}
		&\ - \int_{\XM}\rho_{t}^{2} \nabla \log \frac{\rho_{t}^{1}}{\rho_{t}^{2}}\cdot \nabla\log\frac{\rho_{t}^{2}}{\alpha_t}\ud x \\
		=&\  \frac{1}{2}\left(\|\nabla \log \rho_{t}^{2} - \nabla \log \rho_{t}^{1} \|_{\rho_{t}^{2}}^2 + \|\nabla \log \rho_{t}^{2} - \nabla \log \alpha_t\|^2_{\rho_{t}^{2}} - \|\nabla \log \rho_{t}^{1} - \nabla \log \alpha_t \|_{\rho_{t}^{2}}^2 \right) \\
		\leq&\ \frac{1}{2}\|\nabla \log \rho_{t}^{1} - \nabla\log \rho_{t}^{2}\|_{\rho_{t}^{2}}^2 + \frac{1}{2}\|\nabla \log \rho_{t}^{2} - \nabla \log \alpha_t\|^2_{\rho_{t}^{2}}.
	\end{align*}
	Consequently, we obtain
	\begin{equation*}
		\fracd{\mathrm{KL}(\rho_{t}^{2}, \alpha_t)}{t} \leq \frac{1}{2}\|\nabla \log \rho_{t}^{2} - \nabla \log \rho_{t}^{1}\|_{\rho_{t}^{2}}^2 - \frac{1}{2}\|\nabla \log \rho_{t}^{2} - \nabla \log \alpha_t\|^2_{\rho_{t}^{2}}.
	\end{equation*}
	Omitting the negative term and integrating from $0$ to $t$ and using Lemma \ref{lemma_main}, we have our result.
\end{proof}
We now present the proof of Theorem \ref{thm_main}.
\begin{proof}[Proof of Theorem \ref{thm_main}]
	Using Theorem 6.15 of \citep{villani2009optimal}, we have for any $t$
	\begin{equation}
		W_2^2(\rho^2(t, \cdot), \alpha(t, \cdot)) \leq 2ld \text{TV}^2(\rho^2(t, \cdot), \alpha(t, \cdot)) \leq ld \text{KL}(\rho^2(t, \cdot), \alpha(t, \cdot)),
	\end{equation}
	where we use the Pinsker's inequality in the second inequality.
	Using the triangle inequality of the Wasserstein-2 distance, Theorem \ref{thm_main} is a direct consequence of Lemma \ref{thm_sys_1_2} and Lemma \ref{lemma_kl}.
\end{proof}

\section{Experiment}
\paragraph{Setup.}
In this section, we showcase the effecacy of our approach for numerically solving the FPE with the example where the initial distribution is Gaussian, i.e. $\alpha_0 = \NM(\mu_0, \Sigma_0)$, and the drifting term is a quadratic function, i.e. $V(x) = (x - \mu_\infty)^\top\Sigma_\infty^{-1}(x - \mu_\infty)$. 
We use this example since we know the analytical solution of the FPE $\alpha(t, x)$ in this specific instance and hence we can explicitly calculate the difference between the learned hypothesis velocity field $f_\theta$ and the ground truth.
Specifically, we know that for any time $t \geq 0$, the solution $\alpha(t, \cdot) = \NM(\mu_t, \Gamma_t^\top\Gamma_t)$ is a Gaussian distribution where $\mu_t$ and $\Gamma_t$ evolve in the following manner 
\begin{equation}
    \frac{\ud \mu_t}{\ud t} = \Sigma^{-1}_\infty(\mu_\infty - \mu_t),\quad \frac{\ud \Gamma_t}{\ud t} = -\Sigma^{-1}_\infty \Gamma_t + {\Gamma_t^{-1}}^\top, \Gamma_0 = \sqrt{\Sigma_0},
\end{equation}
if we take the the domain $\XM = \RBB^2$ (see for example Eq.~(36) and Eq.~(37) in \cite{liu2020neural}). 
In our experiment, we take $\mu_0 = (-4, -4) $, $\Sigma_0 = \mathrm{diag}(0.7, 1.3)$, and $\mu_\infty = (4, 4)$, $\Sigma_\infty = \mathrm{diag}(1.1, 0.9)$.

\paragraph{Performance Metrics.} We grid the box $[-10, 10]^2$ with a uniform increment of 0.1 over both coordinates. This gives us $201^2=40401$ grid points altogether and we use $\beta$ to denote the uniform distribution over these points. 
We then grid the time interval $[0, 3]$ with a uniform increment of $0.3$. This gives us $11$ distinct time stamps and we use $\gamma$ to denote the uniform distribution over these time stamps.
Define the score estimation error of a hypothesis velocity field $f$ to be $\ell_s(f) = \int \|f(t, x) + \nabla \log\alpha(t, x) + \nabla V(x)\|^2 \mathrm{d} \beta(x) \mathrm{d} \gamma(t)$, where we note that $(-\nabla V - \nabla \log \alpha)$ is the ground truth velocity field.
Additionally, define the density estimation error of a hypothesis density trajectory $\rho$ as $\ell_d(\rho) = \int |\alpha(t, x) - \rho(t, x)| \mathrm{d} \beta(x) \mathrm{d} \gamma(t)$.
We use the these two quantities in our experiment to measure the quality of the recovered solutions from \texttt{NWGF} (our approach) and we include a successful NN-based PDE solver $\texttt{PINN}$ \citep{raissi2019physics} as the baseline.
Note that the implementation of the continuous time $\texttt{PINN}$ model requires a collection of spatial points $\{x_i\}$ for defining the objective loss, which are set to the grid points mentioned above.

\begin{figure} 
\centering
\begin{tabular}{c c c}
    \includegraphics[width=0.3\textwidth]{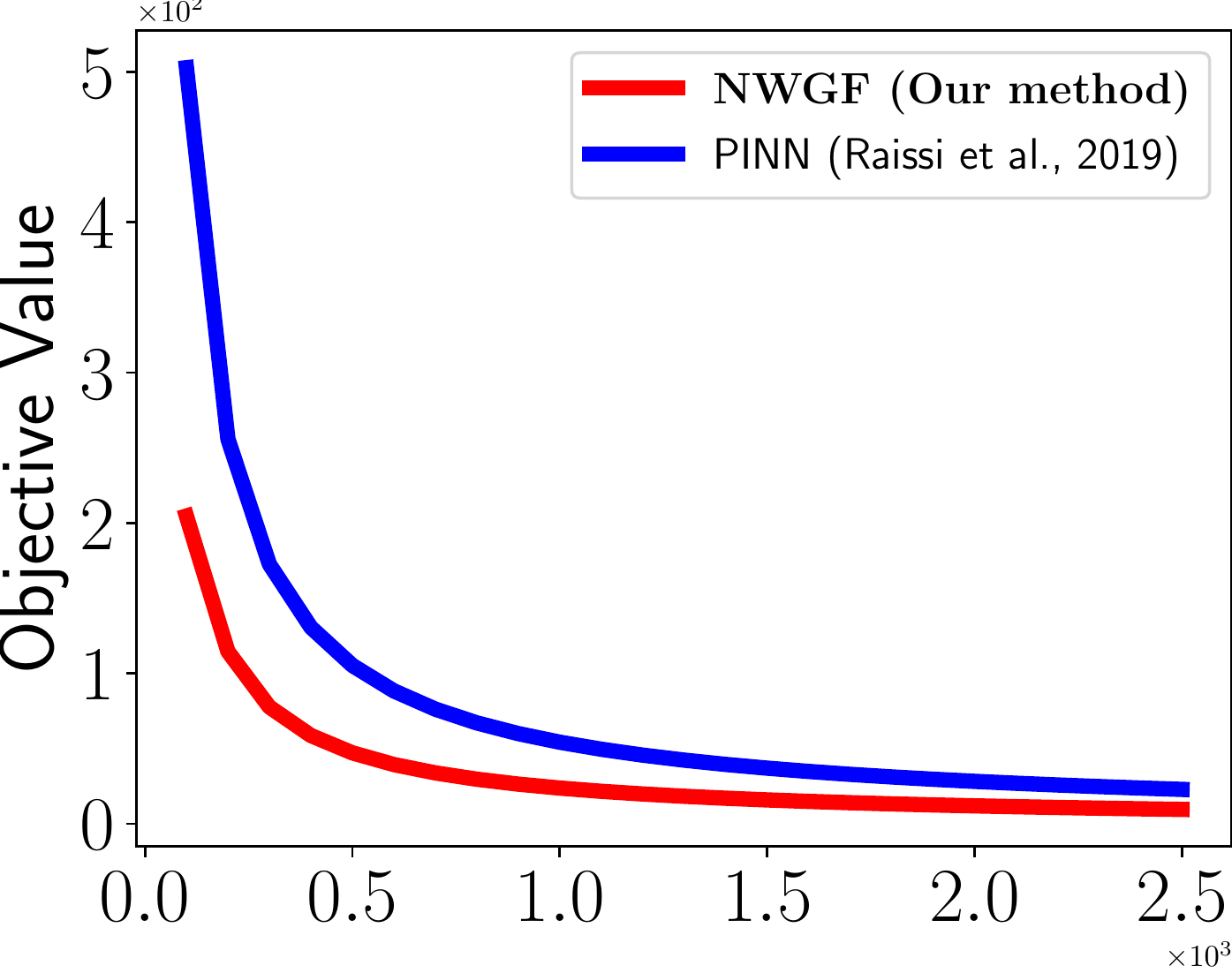} & \includegraphics[width=0.3\textwidth]{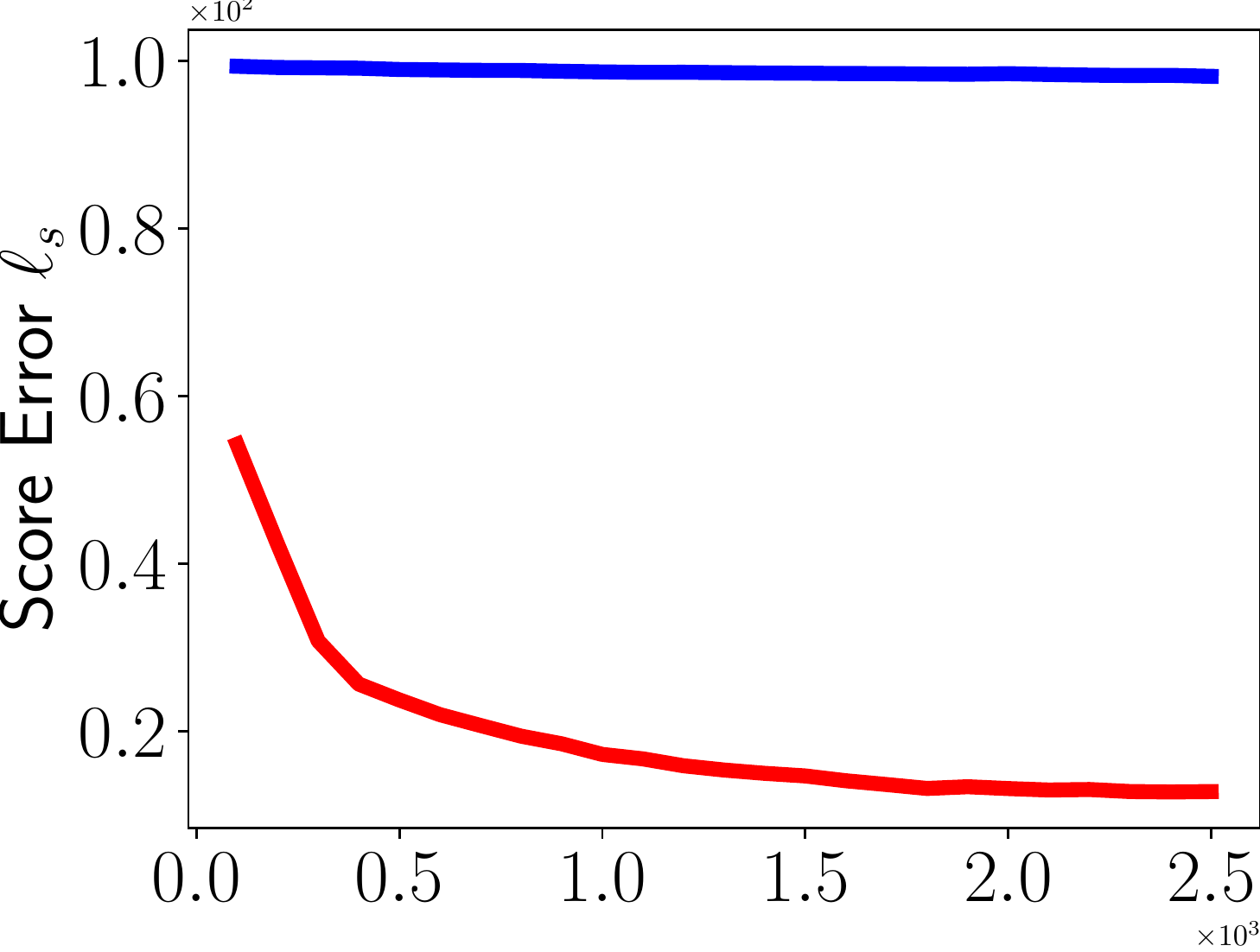} & \includegraphics[width=0.3\textwidth]{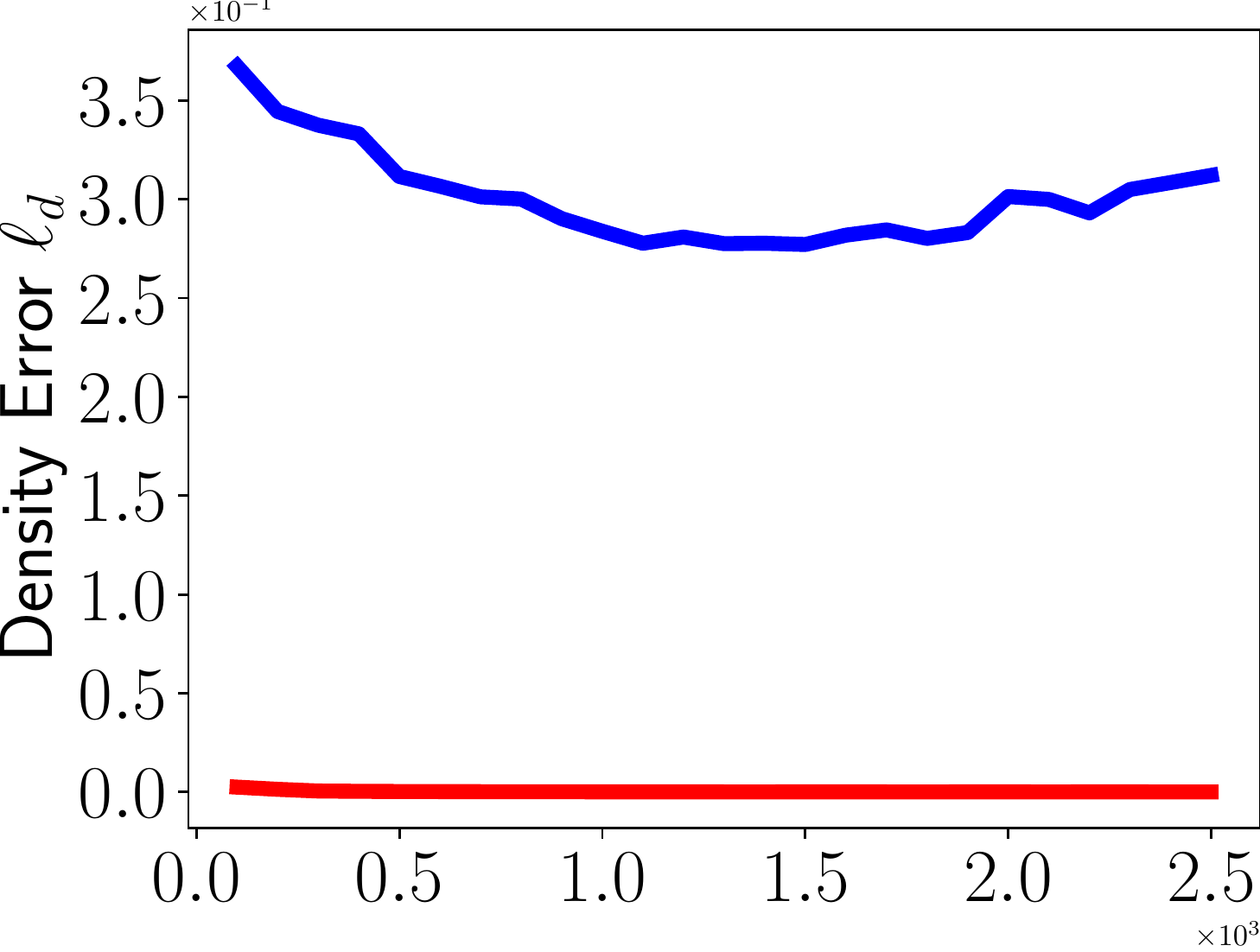} \\
    (i) Objective Value & (ii) Score Estimation Error & (iii) Density Estimation Error
\end{tabular}
\caption{Learning the FPE with a Gaussian initial distribution $\alpha_0$ and a quadratic drifting term $V$.}
\label{fig:result}
\end{figure}

\paragraph{Details.} To avoid negative density values in \texttt{PINN}, instead of directly approximating $\alpha(t, x)$, we use a neural network $g_\theta(t, x)$ to approximate the ground-truth log-density trajectory $\log \alpha$ in \texttt{PINN}. For a fair comparison, the network structures of $f_\theta$ (the hypothesis velocity field used in our approach) and $g_\theta$ are identical except the last layer since $g_\theta$ outputs a scalar (log-density) while $f_\theta$ outputs a $2d$-vector (velocity). We use $\ell_s(-\nabla V -\nabla_x g_\theta)$ to measure the quality of $g_\theta$ as $(-\nabla V -\nabla g_\theta)$ is the hypothesis velocity field that corresponds to $g_\theta$.
We use the strategy discussed in Eq.~\eqref{eqn_recover_density_from_velocity} to reover the density from our hypothesis velocity field $f_\theta$.

\paragraph{Results.}
We report the results of our experiment in Figure \ref{fig:result} and we use \texttt{NWGF} (short for Neural Wasserstein Gradient Flow) to denote our approach.
In plot (i), we observe that stochastic gradient descent is able to reduce the objective values of both \texttt{NWGF} and \texttt{PINN} substantially over $2500$ steps.
However, in plots (ii) and (iii), we observe that our method correctly learns the underlying velocity field and the density trajectories, but these two metrics of $\texttt{PINN}$ barely improve after a long training procedure.
This shows the advantage of our approach.
\section{Conclusion}
In this work, instead of directly approximating the solution to the FPE, we proposed  a learning paradigm  that recovers the entire velocity field, thus understanding better the evolution of the system.
By introducing a velocity-consistency transformation $\AM$ induced by the FPE, we identified a fundamental property of the system called the self-consistency of the FPE. In words, it states that the underlying velocity field of the FPE must be a fixed point of $\AM$.
Based on this novel observation, we designed a potential function $R(f)$ for any hypothesis velocity field $f$ and proved that $R(f)$ controls the Wasserstein-2 distance between the trajectory of distributions generated by $f$ and the exact solution to the FPE.
When the hypothesis velocity field is parameterized by a time-varying neural network, we showed that the stochastic gradient of the proposed potential function with respect to the parameter of the neural network can be computed using the adjoint method.\\

\acks{
The research of Hassani and Shen is supported by NSF Grants 1837253, 1943064, AFOSR Grant FA9550-20-1-0111, DCIST-CRA, and the AI Institute for Learning-Enabled Optimization at Scale (TILOS).
Zhenfu Wang is supported by the National Key R\&D Program of China, Project Number 2021YFA1002800, NSFC grant No.12171009, Young Elite Scientist Sponsorship Program by China Association for Science and Technology (CAST) No. YESS20200028 and the start-up fund from Peking University.
Amin Karbasi acknowledges funding in direct support of this work from NSF (IIS-1845032), ONR (N00014-19-1-2406),  NSF (2112665), and the AI Institute for Learning-Enabled Optimization at Scale (TILOS).

}

\bibliography{colt2022}

\appendix
\section{Proof of Lemma \ref{lemma_change_of_variables}} \label{appendix_proof_lemma_change_of_variable}
\begin{proof}
	From the change-of-variables formula of the pushforward measure, we have
	\begin{equation}
		\int_\XM g \ud X\sharp \alpha = \int_{X^{-1}(\XM)} g \circ X \ud \alpha.
	\end{equation}
	Let $\Pi: \RBB^d \rightarrow \XM$ be that modulus operator such that given any input $x\in\RBB^d$, $\Pi(x)$ is the unique element in $\XM$ such that
	\begin{equation}
		x = \Pi(x) + \sum_{i=1}^{d} n_i l \cdot e_i,
	\end{equation}
	for some $n_i\in\ZBB$, $i = 1, \ldots, d$.
	In the following, we show that (1) $X^{-1}(\XM)$ does not overlap with itself under the operator $\Pi$, i.e. there do not exist two points $x_1, x_2 \in X^{-1}(\XM)$ with $x_1 \neq x_2$ such that $\Pi(x_1) = \Pi(x_2)$ and (2) $\Pi(X^{-1}(\XM)) = \XM$.
	Suppose that these two statements hold, we have
	\begin{equation}
		\int_{X^{-1}(\XM)} g\circ X \ud \alpha \stackrel{(1)}{=} \int_{\Pi(X^{-1}(\XM))} g\circ X \ud \alpha \stackrel{(2)}{=} \int_{\XM} g\circ X \ud \alpha.
	\end{equation}
	To prove (1), suppose that there exist $x_1, x_2 \in X^{-1}(\XM)$ with $x_1 \neq x_2$ such that $\Pi(x_1) = \Pi(x_2)$. There must exist $m_i\in\ZBB$, $i\in\{1, \ldots, d\}$ such that $x_1 = x_2 + (\cdots, m_i \times l, \cdots)$ and that $m_i$'s cannot be all zeros. 
	Since $X(x) - x$ is $l$-periodic, we have
	\begin{equation}
		X(x_1) - x_1 = X(x_2) - x_2 \Rightarrow X(x_1) = X(x_2) + (\cdots, m_i \times l, \cdots).
	\end{equation}
	Since at least one of the $m_i$'s are non-zero, it is impossible that $X(x_1)$ and $X(x_2)$ belong to $\XM$ simultaneously, which leads to a contradiction.\\
	To prove (2), we first observe that $\Pi(X^{-1}(\XM)) \subseteq \XM$ holds trivially due to the definition of $\Pi$, and hence we just need to show that $\XM \subseteq \Pi(X^{-1}(\XM))$.
	We prove via contradiction.
	Suppose that there exists $y \in \XM$ such that $y\notin \Pi(X^{-1}(\XM))$.
	From the definition of the operator $\Pi$, we can write
	\begin{equation} \label{eqn_lemma_change_of_variables}
		X(y) = \Pi(X(y)) + (\cdots, m_i \times l, \cdots),
	\end{equation}
	for some $m_i\in\ZBB$, $i\in\{1, \ldots, d\}$.
	Since $X^{-1}$ is periodic, we have that
	\begin{equation}
		y - (\cdots, m_i \times l, \cdots) = X^{-1}(X(y) - (\cdots, m_i \times l, \cdots)) \stackrel{\eqref{eqn_lemma_change_of_variables}}{=} X^{-1}(\Pi(X(y))) \in X^{-1}(\XM).
	\end{equation}
	However, the above statement means $y \in \Pi(X^{-1}(\XM))$ which contradicts to the definition of $y$.
\end{proof}
\section{Proof of Proposition \ref{prop_score}}
\begin{proof}
	First, compute that
	\begin{align*}
		\fracd{ }{ t} \nabla \log \rho_{t}^{1}(x(t)) =&\ \fracpartial{}{t} \nabla \log \rho_{t}^{1}(x(t)) + \fracd{x(t)}{t} \fracpartial{}{x} \nabla \log \rho_{t}^{1}(x(t))   \\
		=&\  \nabla \fracpartial{}{t} \log \rho_{t}^{1}(x(t)) + f_{t}(x(t)) \nabla^2 \log \rho_{t}^{1}(x(t))  .
	\end{align*}
	Using the Fokker Planck equation \eqref{eqn_FPE}, we derive
	\begin{equation}
		\fracpartial{}{t} \log \rho_{t}^{1} = -\udiv f_{t} - \nabla \log \rho_{t}^1 \cdot f_{t},
	\end{equation}
	which together with 
	\begin{align*}
		\nabla (\nabla \log \rho_{t}^1 \cdot f_{t}) =    \nabla^2 \log \rho_{t}^1 f_{t} +  
		\left(\nabla{f_{t}}\right)^\top \nabla \log \rho_{t}^1  
	\end{align*}
	allows us to compute
	\begin{align*}
		\fracd{ }{ t} \nabla \log \rho_{t}^{1}(x(t)) = - \nabla \udiv f_{t}(x(t))  -  \left(\nabla{f_{t}}(x(t))\right)^\top \nabla \log \rho_{t}^1(x(t))  .
	\end{align*}
	In the above computation, we use the fact that the term $ \nabla^2 \log \rho_{t}^{1}(x(t)) f_{t}(x(t)) $ is canceled.
\end{proof}

\section{Proof of Proposition \ref{prop_Hessian_log}}
\begin{proof}
	For compactness, we use $\partial_{i,j}$ to denote $\frac{\partial^2}{\partial x_i \partial x_j}$.
	First, compute that
	\begin{align*}
		\fracd{ }{ t} \partial_{i,j} \log \rho^{1}_{t, \theta}(x(t)) =&\ \fracpartial{}{t} \partial_{i,j} \log \rho_{t}^{1}(x(t)) + \fracpartial{}{x} \partial_{i,j} \log \rho_{t}^{1}(x(t)) \cdot \fracd{x(t)}{t} \\
		=&\  \partial_{i,j} \fracpartial{}{t} \log \rho_{t}^{1}(x(t)) + \fracpartial{}{x} \partial_{i,j} \log \rho_{t}^{1}(x(t)) \cdot f_{t}(x(t)).
	\end{align*}
	Using the Fokker Planck equation \eqref{eqn_FPE}, we derive
	\begin{equation}
		\fracpartial{}{t} \log \rho_{t}^{1} = -\udiv f_{t} - \nabla \log \rho_{t} \cdot f_{t},
	\end{equation}
	which together with 
	\begin{align*}
		\partial_{i,j} (\nabla \log \rho_{t}^1 \cdot f_{t}) =
		 \partial_{i,j} \nabla \log \rho_{t}^1 \cdot f_{t} + \partial_i\nabla \log \rho_{t}^1 \cdot \partial_j f_{t} \\
		 + \partial_i f_{t}\cdot \partial_j \nabla \log \rho_{t}^1 +  \partial_{i,j} f_{t} \cdot \nabla \log \rho_{t}^1
	\end{align*}
	allows us to compute
	\begin{align*}
		\fracd{ }{ t} \partial_{i,j} \log \rho^{1}_{t, \theta}(x(t)) = - \partial_{i,j} \udiv f_{t}(x(t))  -  \partial_i \nabla \log \rho_{t}(x(t))\cdot \partial_j f_{t}(x(t)) \\
		- \partial_i f_{t}(x(t))\cdot \partial_j \nabla \log \rho_{t}(x(t)) -  \partial_{i,j}f_{t}(x(t)) \cdot \nabla \log \rho_{t}(x(t)).
	\end{align*}
	In the above computation, we use the fact that the term $\fracpartial{}{x} \partial_{i,j} \log \rho_{t}^{1}(x(t)) \cdot f_{t}(x(t))$ is canceled.
\end{proof}
\section{Proof of Proposition \ref{prop_3rd_gradient_log}}
\begin{proof}
		For compactness, we use $\partial_{i,j,k}$ to denote $\frac{\partial^3}{\partial x_i \partial x_j \partial x_k}$.
	First, compute that
	\begin{align*}
		\fracd{ }{ t} \partial_{i,j,k} \log \rho^{1}_{t, \theta}(x(t)) =&\ \fracpartial{}{t} \partial_{i,j,k} \log \rho^{1}(t, x(t); \theta) + \fracpartial{}{x} \partial_{i,j,k} \log \rho^{1}(t, x(t); \theta) \cdot \fracd{x(t)}{t} \\
		=&\ \partial_{i,j,k} \fracpartial{}{t}\log \rho^{1}(t, x(t); \theta) + \fracpartial{}{x} \partial_{i,j,k} \log \rho^{1}(t, x(t); \theta) \cdot f_{t}(x(t)).
	\end{align*}
	Using the Fokker Planck equation \eqref{eqn_FPE}, we derive
	\begin{equation}
		\fracpartial{}{t} \log \rho_{t}^{1} = -\udiv f_{t} - \nabla \log \rho_{t} \cdot f_{t},
	\end{equation}
	which together with 
	\begin{align*}
		\partial_{i,j,k} (\nabla \log \rho_{t}^1 \cdot f_{t}) =  \partial_{i,j,k}\nabla \log \rho_{t}^1 \cdot f_{t} + \partial_{i,j} \nabla \log \rho_{t}^1 \cdot \partial_{k} f_{t} \\
		\partial_{i,k}\nabla \log \rho_{t}^1 \cdot \partial_{j} f_{t} + \partial_{i} \nabla \log \rho_{t}^1 \cdot \partial_{j, k} f_{t} \\
		\partial_{j,k}\nabla \log \rho_{t}^1 \cdot \partial_{i} f_{t} + \partial_{j} \nabla \log \rho_{t}^1 \cdot \partial_{i, k} f_{t} \\
		\partial_{k}\nabla \log \rho_{t}^1 \cdot \partial_{i,j} f_{t} + \nabla \log \rho_{t}^1 \cdot \partial_{i, j, k} f_{t}
	\end{align*}
	allows us to compute
	\begin{align*}
		\fracd{ }{ t}\partial_{i,j,k} \log \rho_{t}^{1}((x(t))) =&\ -\partial_{i,j, k}\udiv f_{t}(x(t)) - \partial_{i,j} \nabla \log \rho_{t}^1(x(t)) \cdot \partial_{k} f_{t}(x(t)) \\
		&\ -\partial_{i,k}\nabla \log \rho_{t}^1(x(t)) \cdot \partial_{j} f_{t}(x(t)) - \partial_{i} \nabla \log \rho_{t}^1(x(t)) \cdot \partial_{j, k} f_{t}(x(t)) \\
		&\ -\partial_{j,k}\nabla \log \rho_{t}^1(x(t)) \cdot \partial_{i} f_{t}(x(t)) - \partial_{j} \nabla \log \rho_{t}^1(x(t)) \cdot \partial_{i, k} f_{t}(x(t)) \\
		&\ -\partial_{k}\nabla \log \rho_{t}^1(x(t)) \cdot \partial_{i,j} f_{t}(x(t)) - \nabla \log \rho_{t}^1(x(t)) \cdot \partial_{i, j, k} f_{t}(x(t)).
	\end{align*}
	In the above computation, we use the fact that the term $\partial_{i,j,k} \nabla \log \rho_{t}^{1}(x(t)) \cdot f_{t}(x(t))$ is canceled.
\end{proof}
\section{Gradient Computation via Adjoint Method} \label{appendix_adjoint_method}
Consider the ODE system
\begin{align*}
	\dot s(t) =&\ \psi(s(t), t, \theta) \\
	s(0) =&\ s_0,
\end{align*}
and the objective loss
\begin{equation}
	\ell(\theta) = \int_0^T g(s(t), t, \theta) \ud t.
\end{equation}
The following proposition computes the gradient of $\ell$ w.r.t. $\theta$.
We omit the parameters of the functions for succinctness. We note that all the functions in the integrands should be evaluated at the corresponding time stamp $t$, e.g. $b^\top \fracpartial{h}{\theta}\ud t$ abbreviates for $b(t)^\top \fracpartial{}{\theta}h(\xi(t), x(t), t, \theta)\ud t$.
\begin{proposition}
	\begin{equation}
		\frac{\ud \ell}{\ud \theta} = \int_{0}^T a^\top \fracpartial{\psi}{\theta} + \fracpartial{g}{\theta}\ud t.
	\end{equation}
	where $a(t)$ is solution to the following final value problems
	\begin{equation}
		\dot a^\top + a^\top \fracpartial{\psi}{s} + \fracpartial{g}{s} = 0, a(T) = 0, 
	\end{equation}
\end{proposition}
\begin{proof}
	Let us define the Lagrange multiplier function (or the adjoint state) $a(t)$ dual to $s(t)$.
	Moreover, let $\LM$ be an augmented loss function of the form
	\begin{equation}
		\LM = \ell - \int_0^T a^\top(\dot s - \psi) \ud t.
	\end{equation}
	Since we have $\dot s(t) = \psi(s(t), t, \theta)$ by construction, the integral term in $\LM$ is always null and $a$ can be freely assigned while maintaining $\ud \LM/\ud \theta = \ud \ell/\ud \theta$.
	Using integral by part, we have
	\begin{equation}
		\int_0^T a^\top\dot s\ \ud t = a(t)^\top s(t)\vert_0^T - \int_0^T s^\top \dot a\ \ud t.
	\end{equation}
	We obtain
	\begin{align}
		\LM = - a(t)^\top s(t)\vert_0^T + \int_0^T \dot a^\top s + a^\top \psi + g\ \ud t.
	\end{align}
	
	Now we compute the gradient of $\LM$ w.r.t. $\theta$ as
	\begin{equation*}
		\frac{\ud \ell}{\ud \theta} =  \frac{\ud \LM}{\ud \theta} = - a(T)^\top\frac{\ud x(T)}{\ud \theta}  + \int_0^T \dot a^\top \frac{\ud s}{\ud \theta} + a^\top \left(\fracpartial{\psi}{\theta} + \fracpartial{\psi}{s} \frac{\ud s}{\ud \theta} \right) \ud t
		+ \int_0^T \fracpartial{g}{s} \frac{\ud s}{\ud \theta} +  \fracpartial{g}{\theta}\ud t,
	\end{equation*}
	which by rearranging terms yields to
	\begin{align*}
		\frac{\ud \ell}{\ud \theta} = \frac{\ud \LM}{\ud \theta} = - a(T)^\top\frac{\ud x(T)}{\ud \theta} + \int_{0}^T a^\top \fracpartial{\psi}{\theta} +  \fracpartial{g}{\theta}\ud t
		+ \int_0^T \left(\dot a^\top + a^\top \fracpartial{\psi}{s} +  \fracpartial{g}{s}\right)\frac{\ud s}{\ud \theta} \ud t.
	\end{align*}
	Now by taking $a$ satisfying the \emph{final} value problems
	\begin{equation}
		\dot a^\top + a^\top \fracpartial{\psi}{s} + \fracpartial{g}{s} = 0, a(T) = 0, 
	\end{equation}
	we derive the result
	\begin{equation}
		\frac{\ud \ell}{\ud \theta} = \int_{0}^T a^\top \fracpartial{\psi}{\theta} + \fracpartial{g}{\theta}\ud t.
	\end{equation}
\end{proof}

\section{Proof of Lemma \ref{lemma_regulairty}} \label{appendix_proof_lemma_regularity}
\begin{proof}	
	Recall the definition of $\rho_{t}^1$ in \eqref{eqn_pushforward_rho1}.
	$\rho_{t}^1$ is $l$-periodic since it can be expressed as a push-forward measure of an $l$-periodic measure $\alpha_0$ under an $l$-periodic map $X(t, \cdot)$.
	Consequently, $\nabla \log \rho_{t}^1$ is also $l$-periodic, which together with the $l$-periodicity of $V$ shows that the map $Y(t, \cdot)$ is also $l$-periodic.
	Following a similar argument, we see that $\rho_{t}^2$ is also $l$-periodic.
	
	To prove that $\|\nabla \log \rho_{t}^1(x)\|$ is bounded for all $x \in \XM$, recall Proposition \ref{prop_score} where we show that for any $x\in\XM$
	\begin{equation}
		\nabla \log \rho_{t}^1(x) = \nabla \log \alpha_0(x(0)) - \int_0^t \nabla \udiv f_{s}(x(s)) + \nabla f_{s}(x(s))^\top\nabla \log \rho_{s}^1(x(s)) \ud s.
	\end{equation}
	Here ${x(s)}_{s \in [0, t]}$ is the trajectory of the final value problem
	\begin{equation} \label{eqn_trajectory_fvp}
		\frac{\ud x(s)}{\ud s} = f_{s}(x(s)), x(t) = x.
	\end{equation}
	Using \gronwall's inequality, we can bound
	\begin{equation}
		\|\nabla \log \rho_{t}^1(x)\| \leq (L_0 + t L_f)\exp(t L_f) \leq (L_0 + T L_f)\exp(T L_f).
	\end{equation}
	To prove that $\|\nabla \log \rho_{t}^1(x)\|$ is Lipschitz continuous for all $x \in \XM$, recall Proposition \ref{prop_Hessian_log} where we show that for any $x\in\XM$
	\begin{align*}
		\nabla^2 \log \rho^{1}_{t}(x) = \nabla^2 \log \alpha_0(x(0)) - \int_0^t\nabla^2 \udiv f_{s}(x(s))  + \left(\nabla^2 \log\rho_{s}^{1}(x(s))\right)^\top \JM_{f_{s}} (x(s)) \qquad \notag \\
		+  \left(\JM_{f_{s}} (x(s))\right)^\top \nabla^2 \log\rho_{s}^{1}(x(s))  + \nabla^2 f_{s}(x(s))\otimes_1 \nabla\log\rho_{s}^{1}(x(s)) \ud s,
	\end{align*}
	where $x(s)$ is the trajectory defined in \eqref{eqn_trajectory_fvp}, $\JM_f$ denotes the Jacobian matrix of a vector valued function $f$, and 
	\begin{equation}
		\nabla^2 f_{s}(x(s))\otimes_1 \nabla\log\rho_{s}^{1}(x(s)) = \begin{bmatrix}
			\nabla^2 (f_{s})_{[1]}(x(s)) \nabla\log\rho_{s}^{1}(x(s)) \\
			\cdots\\
			\nabla^2 (f_{s})_{[d]}(x(s)) \nabla\log\rho_{s}^{1}(x(s))
			\end{bmatrix} \in \RBB^{d \times d}.
	\end{equation}
	Here $f_{[i]}$ denotes the $i$th entry of a vector valued function $f$.
	We can bound the spectral norm $\|\nabla^2 \log \rho_{t}^1(x)\|_{op}$ by (note that $x(t) = x$)
	\begin{align*}
		\|\nabla^2 \log \rho_{t}^1(x(t))\|_{op} \leq&\ L_0 + \int_{0}^{t} L_f + 2L_f\|\nabla^2 \log \rho_{s}^1(x(s))\|_{op} + L_f B_1 \ud s \\
		=&\ L_0 + t(L_f + L_f B_1) + \int_0^t 2L_f \|\nabla^2 \log \rho_{s}^1(x_s)\| \ud s,
	\end{align*}
	where we denote $B_1 = (L_0 + T L_f)\exp(T L_f)$.
	Use \gronwall's inequality to derive
	\begin{equation} \label{eqn_bounded_Hessian_log_rho_1}
		\|\nabla^2 \log \rho_{t}^1(x)\|_{op} \leq (L_0 + t (L_f + B_1L_f))\exp(2t L_f) \leq (L_0 + T(L_f + B_1L_f))\exp(2T L_f),
	\end{equation}

	To see that $\|\nabla \AM[f]_{t}\|_{op}$ is bounded over $\XM$, observe that
	\begin{equation}
		\nabla \AM[f]_{t} = - \nabla^2 V_t - \nabla^2 \log \rho_{t}^1,
	\end{equation}
	which is bounded due to Assumption \ref{ass_regularity_of_v} and \eqref{eqn_bounded_Hessian_log_rho_1}. 
	To see that $\nabla \AM[f]_{t}$ is Lipschitz continuous, we need to prove that the spectral norm of the following tensor is bounded
	\begin{equation}
		\nabla^2 \AM[f]_{t} = - \nabla^3 V_t - \nabla^3 \log \rho_{t}^1.
	\end{equation}
	The first term is  bounded due to Assumption \ref{ass_regularity_of_v}.
	To bound the second term, use Proposition \ref{prop_3rd_gradient_log} to bound (note that $x(t) = x$)
	\begin{align*}
		\|\nabla^3 \log&\ \rho_{t}^1(x(t))\|_{op} \leq \|\nabla^3 \log \alpha_0(x(0))\|_{op} \\
		&\ + \int_0^t\|\nabla^3 \udiv f_{s}(x(s))\|_{op}  + 3\|\nabla^2 f_{s}(x(s))\|_{op}\|\nabla^2 \log\rho_{s}^{1}(x(s))\|_{op} \\
		&\ + 3\|\nabla f_{s}(x(s))\|_{op}\|\nabla^3 \log\rho_{s}^{1}(x(s))\|_{op} + \|\nabla\log\rho_{s}^{1}(x(s))\| \|\nabla^3 f_{s}(x(s))\|_{op} \ud s,
	\end{align*}
	Using \gronwall's inequality, we can bound
	\begin{align*}
		\|\nabla^3 \log \rho_{t}^1(x)\|_{op} \leq (L_0 + t (L_f + B_2L_f + B_1L_f))\exp(3t L_f) \\
		\leq (L_0 + T (L_f + B_2L_f + B_1L_f))\exp(3T L_f),
	\end{align*}
	where we denote $B_2 = 3(L_0 + T(L_f + B_1L_f))\exp(2T L_f)$.
	
	The boundedness of $\|\nabla \udiv \AM[f]_{t}(x)\|$ and the Lipschitz continuity of $\nabla \udiv \AM[f]_{t}$ hold following the same argument above under the assumptions \ref{ass_regularity_0} to \ref{ass_regularity_of_v}.
\end{proof}
\subsection{Proof of Lemma \ref{thm_sys_1_2}} \label{appendix_proof_sys_1_2}
\begin{proof}
	In this proof, for simplicity of the notation, we use $\rho_{t}^{1}$ and $\rho_{t}^{2}$ to denote the probability density functions of systems (1) and (2) and use 
	$X_{t}$ and $Y_{t}$ to denote the corresponding particle maps.
	
	The Wasserstein-2 metric between $\rho_{t}^{1}$ and $\rho_{t}^{2}$ can be written as:
	\begin{equation*}
		W_2^2(\rho_{t}^{1}, \rho_{t}^{2}) = \inf_{P:\ P_\sharp \rho_{t}^{1} = \rho_{t}^{2}} \int_\XM \|x - P(x)\|^2 \ud \rho_{t}^{1}(x),
	\end{equation*}
	where the infimum is taken over all the pushforward maps $P$ such that $P_\sharp \rho_{t}^{1} = \rho_{t}^{2}$.
	From the Lipschitz continuity of the velocity field $f$ in Assumption \ref{ass_regularity_of_f}, the particle map $X_t$ of System (1) is invertible.
	Moreover, recall that Systems (1) and (2) have the same initial distribution $\alpha_0$.	
	We have an upper bound on $W_2^2(\rho_{t}^{1}, \rho_{t}^{2})$ by considering a special map $P_{t, \theta} = Y_{t}\circ X_{t}^{-1}$, where we use $X_{t}$ and $Y_{t}$ to denote the particle maps of systems (1) and (2) compactly (see Table \ref{table_ss}).
	We have the feasibility of $P_{t, \theta}$ by the definitions of $\rho_{t}^{1}$ and $\rho_{t}^{2}$, 
	\begin{equation}
		P_{t, \theta}\sharp \rho_{t}^{1} = {Y_{t}}\sharp( X_{t}^{-1}\circ X_{t}) \sharp \alpha_0 =  \rho_{t}^{2}.
	\end{equation}
	
	Additionally, we have that $\|x - P_{t, \theta}(x)\|$ is $l$-periodic: 
	\begin{align*}
		\|x + le_i - P_{t, \theta}(x + le_i)\| =&\ \|x + le_i - Y_{t} \circ X_{t}^{-1}(x + le_i)\| \stackrel{(1)}{=} \|x + le_i - Y_{t} (X_{t}^{-1}(x) + le_i)\|\\
		\stackrel{(2)}{=}&\ \|x + le_i - (Y_{t} (X_{t}^{-1}(x)) + le_i)\| = \|x - P_{t, \theta}(x)\|,
	\end{align*}
	where in (1) we use $X_{t}^{-1}(x + le_i) = X_{t}^{-1}(x) + le_i$ since
	\begin{align*}
		&\ X_{t}(X_{t}^{-1}(x + le_i) - le_i) - (X_{t}^{-1}(x + le_i) -le_i) = X_{t}(X_{t}^{-1}(x + le_i)) - X_{t}^{-1}(x + le_i) \\
		\Leftrightarrow&\  X_{t}(X_{t}^{-1}(x + le_i) - le_i) = x \quad\quad \Rightarrow X_{t}^{-1}(x + le_i) = X_{t}^{-1}(x) + le_i,
	\end{align*}
	and in (2) we use $Y_{t}(a + le_i)= Y_{t}(a) + le_i$ following a similar argument.
	Therefore, we can bound
	\begin{align*}
		W_2^2(\rho_{t}^{1}, \rho_{t}^{2}) \leq \int_\XM \|x - P_{t, \theta}(x)\|^2 \ud \rho_{t}^{1}(x) =&\ \int_\XM \|X_{t}(x) - Y_{t}(x)\|^2 \ud \alpha_0(x) \\
		=&\ \int_\XM \|x_t - y_t\|^2 \ud \alpha_0(x_0),
	\end{align*}
	where we used the change-of-variables formula of the push-forward measure from Lemma \ref{lemma_change_of_variables} in the first equality and $\{x_t\}_{t\in[0, T]}$ and $\{y_t\}_{t\in[0, T]}$ are the trajectory of particles initialized from $x_0$ but driven by Systems (1) and (2) respectively.
	Hence, we can bound the Wasserstein-2 distance between the trajectory of probability distributions by studying the distance between the particles driven by the two systems, which is proved to be bound by $R(f)$ in expectation ($x_0 \sim \alpha_0$) in the following.

	Suppose two particles are initialized from the same position $x_0$, but follow System (1) and System (2) respectively.
	The change of their distance at time $t$ can be computed by
	\begin{align*}
		&\ \frac{d}{dt}\|x_t - y_t\|^2 = 2 \left(x_t - y_t\right)^\top(\fracd{x_t}{t} - \fracd{y_t}{t})
		=  2 \left(x_t - y_t\right)^\top (f(t,  x_t) - \AM[f](t,  y_t)) \\
		=&\  2 \left(x_t - y_t\right)^\top \left(f(t,  x_t) - \AM[f](t,  x_t)\right) + 2 \left(x_t - y_t\right)^\top \left(\AM[f](t,  x_t) - \AM[f](t,  y_t)\right) \\
		\leq &\ 2\|x_t - y_t\|^2 + \|f(t,  x_t) - \AM[f](t,  x_t)\|^2 + \|\AM[f](t,  x_t) - \AM[f](t,  y_t)\|^2,
	\end{align*}
	where $\AM[f]$ is the velocity field of System (2) and the transformation $\AM$ is defined in equation \eqref{eqn_transform_A}.
	Bound the the last term on the RHS can be bounded by $L^2_v\|x_t - y_t\|^2$ using the Lipschitz continuity of $\AM[f]$ in Lemma \ref{lemma_regulairty} to derive
	\begin{align*}
		&\ \frac{d}{dt}\|x_t - y_t\|^2 \leq (2+L_v^2)\|x_t - y_t\|^2 + \|f(t,  x_t) - \AM[f](t,  x_t)\|^2 \\
		\Rightarrow&\ \frac{d}{dt} \exp(-t(2+L_v^2)) \|x_t - y_t\|^2 \leq \exp(-t(2+L_v^2)) \|f(t,  x_t) - \AM[f](t,  x_t)\|^2
	\end{align*}
	
	Integrate from $t=0$ to $\tau$. By noting that $x_0 = y_0$ and $\exp\left(-(2+L_v^2)t\right) < 1$, we have 
	\begin{align*}
		\exp\left(-(2+L_v^2)\tau\right)\|x_\tau - y_\tau\|^2 \leq
		\int_{0}^{\tau}  \|f(t,  x_t) - \AM[f](t,  x_t)\|^2 \ud t .
	\end{align*}
	Take expectation with respect to $x_0 \sim \alpha_0$. We derive that for any $\tau\in[0, T]$
	\begin{equation} \label{eqn_proof_Wasserstein}
		W_2^2(\rho_{t}^{1}, \rho_{t}^{2}) \leq \int_\XM \|x_\tau - y_\tau\|^2 \ud \alpha_0(x_0) \leq \exp\left((2+L_v^2)T\right) R(f).
	\end{equation}
\end{proof}	

\subsection{Proof of Lemma \ref{lemma_main}} \label{appendix_proof_lemma_score}
\begin{proof}
	For compactness, in this proof, we denote $f_{t}(x) = f(t, x)$ and $\AM[f]_t(x) = \AM[f](t, x)$.
	We use $\rho_{t}^{1}$ and $\rho_{t}^{2}$ to denote the probability density functions of systems (1) and (2) and use 
	$X_{t}$ and $Y_{t}$ to denote the corresponding particle maps (see Table \ref{table_ss}).
	
	Since both $\nabla \log \rho_{t}^1$ and $\nabla \log \rho_{t}^{2}$ are $l$-periodic, using the change of variable formula in Lemma \ref{lemma_change_of_variables} , we have
	\begin{align}
		\xi_t = \|\nabla\log\rho_{t}^{1}\circ Y_{t} - \nabla\log\rho_{t}^{2}\circ Y_{t}\|_{\alpha_0}^2
	\end{align}
	Denote $y_t = Y_{t}(y_0)$ and $x_t = X_{t}(x_0)$ with $y_0 = x_0$. For any $x_0$, we have
	\begin{align}
		\left(\nabla\log\rho_{t}^{1}\circ Y_t\right)(x_0) = \left(\nabla\log\rho_{t}^{1}(y_t) - \nabla\log\rho_{t}^{1}(x_t)\right) + \nabla\log\rho_{t}^{1}(x_t).
	\end{align}
	Hence $\xi_t$ can be bounded by
	\begin{align}
		\xi_t \leq \|\nabla\log\rho_{t}^{1}(y_t) - \nabla\log\rho_{t}^{1}(x_t)\|^2_{\alpha_0} + \|\nabla \log \rho_{t}^{1}(x_t) - \nabla \log \rho_{t}^{2}(y_t)\|^2_{\alpha_0}.
	\end{align}
	The first term is of the order $O(\|x_t - y_t\|^2_{\alpha_0^{2}})$ from the Lipschitz continuity of $\nabla\log\rho_{t}^{1}$.
	To bound the second term, note that $\nabla\log\rho_{t}^{1}(x_t)$ can be computed from from Proposition \ref{prop_score},
	\begin{equation*}
		\nabla \log \rho_{t}^{1}(x_t) = \nabla \log \alpha_0(x_0) - \int_0^t \nabla{\udiv\left(f_{\tau}(x_\tau)\right)} + \left[\nabla{f_{\tau}(x_\tau)}\right]^\top \nabla \log \rho_{t}^{1}(x_\tau) \ud \tau
	\end{equation*}
	and that $\left(\nabla\log\rho_{t}^{2}\circ Y_t\right) (y_0) = \nabla \log \rho_{t}^{2}(y_t)$ can be similarly computed as
	\begin{equation*}
		\nabla \log \rho_{t}^{2}(y_t) = \nabla \log \alpha_0(y_0) - \int_0^t \nabla{\udiv\left(\AM[f]_{\tau}(y_\tau)\right)} + \left[\nabla{\AM[f]_{\tau}(y_\tau)}\right]^\top \nabla \log \rho_{t}^{2}(y_\tau) \ud \tau.
	\end{equation*}
	Hence, the second term can be decomposed as follows:
	\begin{align*}
		\nabla \log \rho_{t}^{1}(x_t) - \nabla \log \rho_{t}^{2}(y_t) = \int_0^t \underbrace{\nabla{\udiv\left(\AM[f]_{\tau}(y_\tau)\right)} - \nabla{\udiv\left(f_{\tau}(x_\tau)\right)}}_{A_\tau} \ud \tau \\
		+ \int_0^t  \underbrace{\left[\nabla{\AM[f]_{\tau}(y_\tau)}\right]^\top \nabla \log \rho_{t}^{2}(y_\tau) - \left[\nabla{f_{\tau}(x_\tau)}x\right]^\top \nabla \log \rho_{t}^{1}(x_\tau)}_{B_\tau} \ud \tau.
	\end{align*}
	Recall that $\delta_{\tau} = f_{\tau} - \AM[f]_{\tau}$ in \eqref{eqn_delta_def}. 
	To bound the norm of $A_\tau$, we have
	\begin{align*}
		A_\tau =  \nabla \udiv\left(\AM[f]_{\tau}(y_\tau)\right) - \nabla \udiv\left(\AM[f]_{\tau}(x_\tau)\right) + \nabla \udiv(\delta_{\tau}(x_\tau))
	\end{align*}
	and hence using $x_\tau \sim \rho_{\tau}^1 = X_{\tau}\sharp\alpha_0$ and the Lipschitz continuity of $\nabla\udiv\AM[f]_t$ we have
	\begin{align*}
		\|A_\tau\|^2_{\alpha_0} = O(\|y_\tau - x_\tau\|^2_{\alpha_0} + \|\nabla \udiv(\delta_{\tau})\|^2_{\rho_{\tau}^{1}}).
	\end{align*}
	To bound the norm of $B_\tau$, note that
	\begin{align}
		B_\tau  =&\  \nabla \AM[f]_{\tau}(y_\tau)^\top \nabla \log \rho_{\tau}^{2}(y_\tau) - \nabla \AM[f]_{\tau}(y_\tau)^\top \nabla \log \rho_{\tau}^{1}(y_\tau) \tag{a} \\
		&\ + \nabla \AM[f]_{\tau}(y_\tau)^\top \nabla \log \rho_{\tau}^{1}(y_\tau) - \nabla \AM[f]_{\tau}(x_\tau)^\top \nabla \log \rho_{\tau}^{1}(y_\tau) \tag{b}\\
		&\ + \nabla \AM[f]_{\tau}(x_\tau)^\top \nabla \log \rho_{\tau}^{1}(y_\tau) - \nabla f_{\tau}(x_\tau)^\top \nabla \log \rho_{\tau}^{1}(y_\tau) \tag{c}\\
		&\ + \nabla f_{\tau}(x_\tau)^\top \nabla \log \rho_{\tau}^{1}(y_\tau) - \nabla f_{\tau}(x_\tau)^\top \nabla \log \rho_{\tau}^{1}(x_\tau) \tag{d}
	\end{align}
	Using the boundedness of $\nabla f_{\tau}$ and the Lipschitz continuity of $\nabla \log \rho_{\tau}^{1}$, we have $\|d\|^2_{\alpha_0} = O(\|x_\tau-y_\tau\|^2_{\alpha_0})$.
	Similarly, we have $\|b\|^2_{\alpha_0} = O(\|x_\tau-y_\tau\|^2_{\alpha_0})$.
	Note that 
	\begin{equation}
		c = - \nabla \delta_{\tau}(x_\tau)^\top \nabla \log \rho_{\tau}^{2}(y_\tau). 
	\end{equation}
	Using the boundedness of $\nabla \log \rho_{t}^{2}$, we have 
	\begin{equation}
		\|c\|^2_{\alpha_0} = O(\|\nabla \delta_{\tau}\|^2_{\rho_{t}^{1}}).
	\end{equation}
	Finally, using the boundedness of $\nabla \AM[f]_{\tau}$, we have that 
	\begin{equation}
		\|a\|^2_{\alpha_0} \leq L_v \|\nabla \log \rho_{\tau}^{2}\circ Y_\tau - \nabla \log \rho_{\tau}^{1}\circ Y_\tau\|^2_{\alpha_0} = L_v \|\nabla \log \rho_{\tau}^{2} - \nabla \log \rho_{\tau}^{1}\|^2_{\rho_{\tau}^{(2)}} = L_v\xi_\tau.
	\end{equation}
	Therefore, by noting that 
	\begin{equation}
		\|\nabla \log \rho_{t}^{1}(x_t) - \nabla \log \rho_{t}^{2}(y_t)\|^2_{\alpha_0} \leq \int_0^t \|A_\tau\|^2_{\alpha_0} + \|B_\tau\|^2_{\alpha_0} \ud \tau,
	\end{equation}
	we bound (note that $\|\delta_{\tau}\|_{\rho_{\tau}^{1}} = \|\delta_{\tau}\circ X_\tau\|_{\alpha_0}$)
	\begin{align*}
		\xi_t \leq &\ \int_0^t O(\|y_\tau - x_\tau\|^2_{\alpha_0} + \|\nabla \udiv(\delta_{\tau})\|^2_{\rho_{\tau}^{1}} + \|\nabla \delta_{\tau}\|^2_{\rho_{\tau}^{1}}) + L_v\xi_\tau \ud \tau \\
		 \leq &\ \int_0^t O(\|\delta_{\tau}\|^2_{\alpha_0} + \|\nabla \udiv(\delta_{\tau})\|^2_{\rho_{\tau}^{1}} + \|\nabla \delta_{\tau}\|^2_{\rho_{\tau}^{1}}) + L_v\xi_\tau \ud \tau\\
		 \leq &\ \int_0^t O(R(f)) + L_v\xi_\tau \ud \tau
	\end{align*}
	where we use Lemma \ref{thm_sys_1_2} in the second inequality.
	Using the Gr\"onwall's inequality of the integral form for continuous functions, we have there exists some constant $\bar C(T)$ such that
	\begin{equation}
		\xi_t \leq \bar C(T) R(f) \exp(tL_v) \leq \bar C(T) R(f) \exp(TL_v) 
	\end{equation}
	Integrating $\tau$ from $0$ to $t$, we have for any $t\in[0, T]$
	\begin{equation}
		\int_0^t \xi_\tau \ud \tau \leq \bar C(T) T \exp(TL_v) R(f) = C(T) R(f),
	\end{equation}
	where we denote $C(T) = \bar C(T) T \exp(TL_v)$.
\end{proof}

%
%
%

\end{document}